\documentclass[letterpaper]{article}

\PassOptionsToPackage{numbers, compress}{natbib}

\usepackage[preprint]{neurips_2024}
\usepackage{todonotes}

\usepackage[utf8]{inputenc} 
\usepackage[T1]{fontenc}    
\usepackage[hidelinks]{hyperref}       
\usepackage{url}            
\usepackage{booktabs}       
\usepackage{amsfonts}       
\usepackage{nicefrac}       
\usepackage{microtype}      
\usepackage{xcolor}         
\usepackage{amsmath} 
\usepackage{amssymb}
\usepackage{mathtools}
\usepackage{amsthm}
\usepackage{graphicx}
\usepackage{subcaption}
\usepackage{algorithm}
\usepackage{algorithmic}
\usepackage[capitalize,noabbrev]{cleveref}

\usepackage{wrapfig}

\theoremstyle{plain}
\newtheorem{theorem}{Theorem}[section]

\newtheorem{lemma}[theorem]{Lemma}
\newtheorem{corollary}[theorem]{Corollary}
\theoremstyle{definition}

\newtheorem{assumption}[theorem]{Assumption}
\theoremstyle{remark}

\DeclareMathOperator*{\argmax}{arg\,max}

\title{Safe Time-Varying Optimization based on Gaussian Processes with Spatio-Temporal Kernel}

\author{%
  Jialin Li 
    \\
 ETH Zurich\\
  \texttt{lijial@ethz.ch} \\
  \And
  Marta Zagorowska \\
  Norwegian University of Science and Technology \\
  \texttt{marta.zagorowska@ntnu.no} \\
  \And
  Giulia De Pasquale \\
  ETH Zurich \\
  \texttt{degiulia@control.ee.ethz.ch} \\
  \And
  Alisa Rupenyan \\
  Zurich University of Applied Sciences \\
  \texttt{alisa.rupenyan@zhaw.ch} \\
  \And
  John Lygeros \\
  ETH Zurich \\
  \texttt{jlygeros@ethz.ch} \\
}

\begin{document}

\maketitle

\begin{abstract}
Ensuring safety is a key aspect in sequential decision making problems, such as robotics or process control. The complexity of the underlying systems often makes finding the optimal decision challenging, especially when the safety-critical system is time-varying. Overcoming the problem of optimizing an unknown time-varying reward subject to unknown time-varying safety constraints, we propose TVS{\small AFE}O{\small PT}, a new algorithm built on Bayesian optimization with a spatio-temporal kernel. The algorithm is capable of safely tracking a time-varying safe region without the need for explicit change detection. Optimality guarantees are also provided for the algorithm when the optimization problem becomes stationary. We show that TVS{\small AFE}O{\small PT} compares favorably against S{\small AFE}O{\small PT} on synthetic data, both regarding safety and optimality. Evaluation on a realistic case study with gas compressors confirms that TVS{\small AFE}O{\small PT} ensures safety when solving time-varying optimization problems with unknown reward and safety functions.
\end{abstract}

\section{Introduction}
\label{sec:Introduction}

We seek to interactively optimize an unknown time-varying reward function $f: \mathcal{X} \times \mathcal{T} \rightarrow \mathbb{R}$, where $\mathcal{X}$ is a finite set of decisions, and $\mathcal{T} := \{0,1,2,\ldots,T\},\; T \in \mathbb{N}_+$ denotes the discretized time domain.  We assume that the optimization problem is safety-critical, that is, there are constraints that evaluated decisions must satisfy with high probability. Similar to the reward, the constraints are also unknown and potentially time-varying, encoded through $c_i: \mathcal{X} \times \mathcal{T} \rightarrow \mathbb{R}, \; i \in \mathcal{I}_c:=\{1,2,\ldots,m\},\; m \in \mathbb{N}_+$. The optimization problem at a given time $t$ can be written as
\begin{align} \label{eq: tv_opt}
\begin{aligned}
\max \limits _{\mathbf{x} \in \mathcal{X}} &\;f(\mathbf{x}, t) \\
\text{  subject to }&c_i(\mathbf{x}, t) \geq 0,\;i \in \mathcal{I}_c   
\end{aligned}    
\end{align}
Both the reward function and the safety constraints are assumed to be unknown but can be evaluated. This is a plausible setting, for example, for UAV that need to perform rescue missions in dangerous and poorly lit environments.

\subsection{Related Work}

Bayesian Optimization (BO) is a well-established approach for interactively optimizing unknown reward functions. Various BO based approaches have been proposed to solve a wide range of problems in robotics \cite{DL-RG-DS:07,RMC-NDF-AD-JAC:07}, combinatorial optimization \cite{ZW-BS-LJ-NDF:13}, sensor networks \cite{NS-AK-SMK-MS:10}, and automatic machine learning  \cite{MH-BS-NDF:14,JS-HL-RPA:12}. However, Safe Bayesian Optimization in the time-varying setting is still under-explored. 

\textbf{Safe Bayesian Optimization} 
\looseness -1
To address safety requirements in safety-critical applications, Safe Bayesian Optimization (SBO) \cite{sui2015safe} has been proposed to avoid unsafe decisions with high probability by interactively optimizing a reward function under safety constraints. S{\small AFE}O{\small PT} \cite{sui2015safe}, one of the first SBO algorithms, expands an initial safe set iteratively based on new evaluations and an updated Gaussian Process (GP) model of safety functions. It calculates two subsets—maximizers and expanders—from the current safe set and selects the most uncertain decision within their union to balance maximizing the reward function and expanding the safe set. Subsequent variants extend S{\small AFE}O{\small PT} to handle multiple constraints \cite{berkenkamp2021bayesian}, decouple safe set expansion from optimization \cite{sui2018stagewise}, and expand the safe set in a goal-oriented manner \cite{turchetta2019safe}. These methods also explore disconnected safe regions \cite{baumann2021gosafe, sukhija2023gosafeopt} and enhance information-theoretic efficiency \cite{bottero2024informationtheoretic, hubotter2024information}. They have been applied to controller tuning for a ball-screw drive \cite{zagorowska2023efficient} and quadrupeds \cite{widmer2023tuning}, and adaptive control on a rotational motion system \cite{koenig2023risk}. However, SBO typically does not take into account changes with time.

\textbf{Contextual Bayesian Optimization} 
\looseness -1
Contextual Bayesian Optimization (CBO) has been introduced to address the influence of external environmental factors on reward and safety functions. \citet{krause2011contextual} extends the Gaussian Process Upper Confidence Bound (GP-UCB) algorithm \cite{srinivas2009gaussian} by incorporating contextual variables into unconstrained BO, demonstrating sub-linear regret analogous to GP-UCB. An advancement of this framework is proposed in \cite{fiducioso2019safe}, with the Safe Contextual GP-UCB optimizing the contextual upper confidence bound within a safe set to manage room temperature via a PID controller. \citet{berkenkamp2021bayesian} presents a contextual adaptation of S{\small AFE}O{\small PT}, discussing its safety and optimality guarantees by framing contextual SBO as distinct SBO sub-problems. Additionally, \citet{konig2021safe} extends G{\small O}OSE \cite{turchetta2019safe} to the contextual domain for model-free adaptive control scenarios. Similarly to SBO, CBO does not explicitly consider time-varying problems. 

\textbf{Time-Varying Bayesian Optimization}
\looseness -1
Time-Varying Bayesian Optimization (TVBO) addresses problems where the objective is time-dependent, modeled with a temporal kernel \cite{bogunovic2016time}. Methods in this setting include periodical resetting \cite{bogunovic2016time}, change detection \cite{brunzema2022event, hong2023optimization}, sliding-window approaches using recent data \cite{zhou2021no}, and discounting via exponentially decaying past observations \cite{deng2022weighted}. However, these techniques have been developed for unconstrained BO and are not suitable for safety-critical applications.

\textbf{Time-Varying Safe Bayesian Optimization}
\looseness -1
In the safety-critical time-varying setting, contextual lower confidence bounds can be optimized within the safe set \cite{fiducioso2019safe}, but it does not guarantee optimality theoretically. An event triggering mechanism is introduced to SBO to restart exploration from a backup policy \cite{holzapfel2023event}, but it may not trigger reliably during changes, posing a safety risk. Extensions to SBO with contextual variables provide theoretical safety and optimality analyses \cite{berkenkamp2021bayesian, widmer2023tuning}, treating contextual SBO as separate sub-problems for each contextual value, and assuming an initial safe set for each. However, ensuring optimality requires each contextual value to appear frequently, which is impractical in time-varying scenarios.

\subsection{Methodology and Contributions   }  
\textbf{Methodology}
\looseness -1
We propose the TVS{\small AFE}O{\small PT} algorithm to  optimize an unknown time-varying reward subject to unknown time-varying safety constraints. The algorithm focuses on Time-Varying Safe Bayesian Optimization (TVSBO). TVS{\small AFE}O{\small PT} utilizes a spatio-temporal kernel and time Lipschitz constants as prior knowledge about how the problem depends on time. The temporal part of the kernel encodes the continuity of the functions with time while the Lipschitz constants explicitly provide upper bounds on how fast the functions may change. Instead of considering safe sets at previous iteration as safe at the current iteration, which might lead to unsafe decisions, TVS{\small AFE}O{\small PT} robustly subtracts the safety margin when updating the safe sets (Figure~\ref{fig:toy_example_safe_set}). In this way, the algorithm is capable of adapting in real time and guarantees safety even when exploring the safe region of non-stationary problems. 

\textbf{Contributions} 
\looseness -1
Our contributions are threefold: a) We propose the TVS{\small AFE}O{\small PT} algorithm based on Gaussian processes with spatio-temporal kernels; b) We provide formal safety guarantees for TVS{\small AFE}O{\small PT} in the most general time-varying setting and  optimality guarantees for TVS{\small AFE}O{\small PT} for locally stationary optimization problems; c) We show TVS{\small AFE}O{\small PT} performs well in the most general time-varying setting both on synthetic data and on a realistic case study on gas compressors. 

\subsubsection{Expected societal impact}
\label{sec:impact}
The TVS{\small AFE}O{\small PT} algorithm proposed in this paper extends the state of the art in Time-Varying Safe Bayesian Optimization by enabling solving optimization problems with time-varying reward and constraints without pre-defining the time changes that can be compensated. As such, the algorithm can be used at the design stage of operating strategies for safety-critical systems, such as medical dosage design \cite{krishnamoorthy2022safe} and controller design in robotics \cite{koenig2023risk}, or during online operation of chemical plants \cite{krishnamoorthy2023model} or autonomous racing \cite{hewing2020learning}.

\section{TVS{\small AFE}O{\small PT} Algorithm}
\label{sec:algorithm}
\looseness -1
The TVS{\small AFE}O{\small PT} algorithm builds upon S{\small AFE}O{\small PT} \cite{sui2015safe}, to handle time-varying reward function and safety functions. The key new feature of TVS{\small AFE}O{\small PT} is its capability of safely transferring the current safe set to the next time step. TVS{\small AFE}O{\small PT} achieves this with the help of the spatio-temporal kernel as well as the sequence of time Lipschitz constants. The approach is summarized in Algorithm~\ref{alg: tv_safeopt}.

\subsection{Assumptions}
\label{sec:ProblemSetting}

Following \cite{berkenkamp2021bayesian}, we wrap the reward and safety functions into an auxiliary function $h: \mathcal{X} \times \mathcal{T} \times \mathcal{I} \rightarrow \mathbb{R}$, where $\mathcal{I} := \{0\} \cup \mathcal{I}_c$,
\begin{align} \label{eq: surrogate}
\begin{aligned}
h(\mathbf{x}, t, i):= \begin{cases}f(\mathbf{x}, t) & \text {, if } i=0 \\ c_i(\mathbf{x}, t) & \text {, if } i \in \mathcal{I}_c\end{cases}    
\end{aligned}    
\end{align}

We model the auxiliary function using a prior Gaussian Process (GP) with zero mean and spatio-temporal kernel $\kappa:  (\mathcal{X} \times \mathcal{T} \times \mathcal{I}) \times (\mathcal{X} \times \mathcal{T} \times \mathcal{I}) \rightarrow \mathbb{R}$, \cite{3569}. To ensure enough regularity, we require $h$ to be Lipschitz continuous with respect to both $\mathbf{x}$ and $t$, and to have bounded norm in the Reproducing Kernel Hilbert Space (RKHS) \cite{scholkopf2002learning} associated with the kernel $\kappa$.

\begin{assumption}\label{ass:regularity}
    The spatio-temporal kernel is positive definite, and satisfies $\kappa\left((\mathbf{x},t,i),(\mathbf{x},t,i)\right) \leq 1$, for all $ \mathbf{x} \in \mathcal{X}, t \in \mathcal{T}, i \in \mathcal{I}$. The function $h(\mathbf{x}, t, i)$ has bounded norm in the RKHS associated with kernel $\kappa$. The function $h(\mathbf{x}, t, i)$ is $L_{\mathbf{x}}$-Lipschitz continuous with respect to $\mathbf{x}$ in the domain $\mathcal{X}$ with respect to some metric $d: \mathcal{X} \times \mathcal{X} \rightarrow \mathbb{R}_{\geq 0}$ for all $t \in \mathbb{N}$, $i \in \mathcal{I}$. There exists a sequence $\{L(t)\}_{ t \in \mathbb{N}, t < T}$, such that, for all $\mathbf{x} \in \mathcal{X}$, $i \in \mathcal{I}$, $t \in \mathbb{N}, t < T$, $|h(\mathbf{x}, t+1, i) - h(\mathbf{x}, t, i)| \leq L(t)$.
\end{assumption} 

At each algorithm iteration $k$, we make a decision $\mathbf{x}_k$, which we then apply to the system and get noisy measurements $y_k^i$ of the reward function and safety functions during the iteration. We use the index $k$ to refer to the algorithm iteration; although $k$ and $t$ might differ in principle, in practice we run one algorithm iteration $k$ for each time step $t$.

\begin{assumption} \label{ass:noise}
    Observations $y_k^i = h(\mathbf{x}_k, t, i)+\varepsilon_k^i,\; \forall i \in \mathcal{I},\; t \in \mathbb{N}$ are perturbed by i.i.d. zero mean and $\sigma$-sub-Gaussian noise. 
\end{assumption}

Based on the measurements, we compute the posterior GP and make the decision for the next time step. To start the exploration, an initial set of safe decisions is assumed to be available to the algorithm. To ensure that the safe set remains non-empty after the first iteration, it is necessary to assume that the initial safety function values at every decision within the initial safe set are positive.

\begin{assumption} \label{ass:S_0}
    An initial set $S_0 \subseteq \mathcal{X}$ of safe decisions is known and for all decisions $\mathbf{x} \in S_0$, we have $c_i(\mathbf{x}, 0) > 0, \; \forall i \in \mathcal{I}_c$.
\end{assumption}

Note that similar assumptions have also been made for the standard  S{\small AFE}O{\small PT} algorithm \cite{sui2015safe} and are necessary to ensure feasibility of the exploration steps and be able to identify new safe decision.

\subsection{Safety Updates}
To ensure safety, based on Assumption~\ref{ass:regularity} and \ref{ass:noise}, we extend the definition of the confidence intervals from \cite{sui2015safe} so that, with high probability, they contain $f$ and $c_i$ using the posterior GP estimate given the data sampled so far. The confidence intervals for $h(\mathbf{x}, t, i)$ given training samples until iteration $k \geq 1$ are defined for all $ \mathbf{x} \in \mathcal{X}$ and for all $i\in \mathcal{I} $ as
\begin{align} \label{def: Q_t}
    \begin{aligned}
        Q_k(\mathbf{x},i) := \left[\mu_{k-1}(\mathbf{x},i) \pm \sqrt{\beta_{k}} \sigma_{k-1}(\mathbf{x},i)\right],\;
    \end{aligned}
\end{align}
where $\beta_{k}$ is a scalar that determines the desired confidence
interval, $\mu_{k-1}(\mathbf{x},i)$ and $\sigma_{k-1}(\mathbf{x},i)$ are the posterior mean and standard deviation of $h(\mathbf{x},t,i)$ inferred with $\mathcal{D}_k$, training samples till iteration $k$ \cite{3569}. The probability of the true function value $h$ lying within this interval depends on the choice
of $\beta_{k}$ \cite{berkenkamp2021bayesian}. We provide more details for this choice in Section \ref{subsec:safety_guarantees}.

It is possible to construct a tighter confidence interval for $h(\mathbf{x}, t, i)$ by using the  sequence $\{Q_\tau(\mathbf{x},i)\}_{\tau \leq k}$ instead of $Q_{k}(\mathbf{x},i)$ alone. To this end, we recursively define for all $\mathbf{x} \in \mathcal{X}$ and for all  $ i\in \mathcal{I}$ the intersection 
\begin{align} \label{def: C_t}
    \begin{aligned}
        C_{k}(\mathbf{x}, i) := \left(C_{k-1}(\mathbf{x},i) \oplus [-L(t-1), L(t-1)] \right) \cap Q_{k}(\mathbf{x},i),
    \end{aligned}
\end{align}
where $\oplus$ denotes the Minkowski sum, $C_0(\mathbf{x}, i)$ is $[L(0), \infty)$ for all $\mathbf{x} \in S_0$, $i \in \mathcal{I}_c$ and $\mathbb{R}$ otherwise. We utilize lower bound $l_{k}(\mathbf{x},i) := \min C_{k}(\mathbf{x},i)$, upper bound $u_{k}(\mathbf{x},i) := \max C_{k}(\mathbf{x},i)$, and width of $C_{k}(\mathbf{x},i)$ 
\begin{equation} \label{def: w_t}
    w_{k}(\mathbf{x},i) := u_{k}(\mathbf{x},i) - l_{k}(\mathbf{x},i)
\end{equation} 
to update the safe set as well as pick the next decision to explore.

Based on the updated posterior and Lipschitz constants, we can update the safe set $S_k$ with the lower bounds $l_k$ and the previous safe set $S_{k-1}$ as
\begin{align} \label{def: S_t}
    \begin{aligned}
        S_k = \cap_{i \in \mathcal{I}_c} \cup_{\mathbf{x} \in S_{k-1}}\{\mathbf{x}^{\prime} \in \mathcal{X} \mid l_k(\mathbf{x},i)
        -L_{\mathbf{x}} d(\mathbf{x}, \mathbf{x}^{\prime}) - L(t) \geq 0\}.
    \end{aligned}
\end{align}
The set $S_k$ contains decisions that with high probability fulfill the safety constraints given the GP confidence intervals and the Lipschitz constants. In contrast to S{\small AFE}O{\small PT}, the safe set of TVS{\small AFE}O{\small PT} is allowed to shrink to adapt to the potential change of the safe region given the time-varying setting. However, the safe set might even become empty after the update. This is either because the safe region indeed becomes empty or because the updated safe set conservatively excludes all decisions with a lower bound of some safety function below $L$ to guarantee safety. In all these cases, if the updated safe set is empty, we terminate the algorithm.

\subsection{Safe Exploration and Exploitation}
With the safe set updated, the next challenge is to trade off between exploitation and expansion of the safe region. As in the standard S{\small AFE}O{\small PT}, the potential maximizers are those decisions, for which the upper confidence bound of the reward function is higher than the largest lower confidence bound, that~is,
\begin{align} \label{def: M_t}
    \begin{aligned}
        M_k = \left\{\mathbf{x} \in S_k \mid u_k(\mathbf{x},0) \geq \max_{\mathbf{x}^{\prime} \in S_k} l_k(\mathbf{x}^{\prime},0)\right\}.
    \end{aligned}
\end{align}
To identify the potential expanders, $G_k$, containing all decisions that could potentially expand the safe set, we first quantify the potential enlargement of the current safe set after sampling a new decision $\mathbf{x}$.  To do so, we define the function
\begin{align} \label{def: e_t}
    \begin{aligned}
        e_k(\mathbf{x}) := 
        |\{\mathbf{x}^{\prime} \in \mathcal{X} \backslash S_k \mid \exists i \in \mathcal{I}_c:
        u_k(\mathbf{x},i) - L_\mathbf{x} d(\mathbf{x},\mathbf{x}^{\prime}) - L(t) \geq 0 \}|,
    \end{aligned}
\end{align}
where $|\cdot|$ refers to the cardinality of a set, and then update 
\begin{align} \label{def: G_t}
    \begin{aligned}
        G_k = \left\{\mathbf{x} \in S_k \mid e_k(\mathbf{x})>0\right\}.
    \end{aligned}
\end{align}

At iteration $k$, TVS{\small AFE}O{\small PT} selects a decision $\mathbf{x}_k$ within the union of potential maximizers \eqref{def: M_t} and expanders \eqref{def: G_t}
\begin{align} \label{def: x_t}
    \begin{aligned}
       \mathbf{x}_k = \argmax_{\mathbf{x} \in G_k \cup M_k, i \in \mathcal{I}} w_k(\mathbf{x},i),
    \end{aligned}
\end{align}
where $w_k(\mathbf{x},i)$ is defined as in (\ref{def: w_t}). The objective of the greedy selection process in \eqref{def: x_t} is to take the most uncertain decision among the expanders $G_k$ and the maximizers $M_k$. The decision $\mathbf{x}_k$ is then applied to the system and after making observations of the reward and safety functions, $\mathbf{y}_k := (y_k^0, y_k^1, \ldots, y_k^m)$, we add $(\mathbf{x}_k, \mathbf{y}_k)$ to the training samples.

At any iteration, we can obtain an estimate for the current best decisions from
\begin{align} \label{def: x_t_hat}
    \begin{aligned}
       \hat{\mathbf{x}}_k = \argmax_{\mathbf{x} \in S_k} l_k(\mathbf{x},0),
    \end{aligned}
\end{align}
which returns the maximizer of the lower bound of the reward function within the current safe set. 

\begin{algorithm}[H] 
  \begin{algorithmic}[1]
  \STATE \textbf{Input:}
  Sample set $\mathcal{X}$\\
  GP priors for $f$, $c_i$\\
  Lipschitz constants $L_{\mathbf{x}}$ and $\{L(t)\}_{t \in \mathbb{N}, t<T}$\\
  Safe set seed $S_0$\\
  \STATE $C_0(\mathbf{x},i) \leftarrow\;[L(0), \infty)$, for all $\mathbf{x} \in S_0,\; i \in \mathcal{I}_c$
  \STATE $C_0(\mathbf{x},i) \leftarrow \mathbb{R}$, for all $\mathbf{x} \in \mathcal{X} \backslash S_0,\; i \in \mathcal{I}_c$
  \STATE$C_0(\mathbf{x},0) \leftarrow \mathbb{R}$
  \STATE Query a point $\mathbf{x}_0 \in S_0$, $y_0^i \leftarrow h(\mathbf{x}_0, 0, i)+\varepsilon_0^i,\; i \in \mathcal{I}$
  \STATE $\mathcal{D}_{0}=\left\{\left(\mathbf{x}_0, \mathbf{y}_0\right)\right\}$
  \FOR{$k = 1, 2, \cdots, T $}
    \STATE Calculate $Q_k(\mathbf{x},i)$ as in \eqref{def: Q_t}, $\forall \mathbf{x} \in \mathcal{X}, \forall i \in \mathcal{I}$
    \STATE $
        C_k(\mathbf{x},i) \leftarrow \left(C_{k-1}(\mathbf{x},i) \oplus  [-L(t-1), L(t-1)] \right) \cap Q_k(\mathbf{x},i)$
    \STATE $
        S_k \leftarrow \cap_{i \in \mathcal{I}_c} \cup_{\mathbf{x} \in S_{t-1}}\{\mathbf{x}^{\prime} \in \mathcal{X} \mid l_k(\mathbf{x},i)
        -L_{\mathbf{x}} d(\mathbf{x}, \mathbf{x}^{\prime}) - L(t) \geq 0\}$ 
    \IF{$S_k = \varnothing$} 
    \STATE break
    \ENDIF
    \STATE $M_k \leftarrow\left\{\mathbf{x} \in S_k \mid u_k(\mathbf{x},0) \geq \max_{\mathbf{x}^{\prime} \in S_k} l_k(\mathbf{x}^{\prime},0)\right\}$
    \STATE $G_k \leftarrow\left\{\mathbf{x} \in S_k \mid e_k(\mathbf{x})>0\right\}$ with $e_k(\mathbf{x})$ from \eqref{def: e_t}
    \STATE $\mathbf{x}_k \leftarrow \argmax_{\mathbf{x} \in G_k \cup M_k, i \in \mathcal{I}} w_k(\mathbf{x},i)$
    \STATE $y_k^i \leftarrow h(\mathbf{x}_k, t, i)+\varepsilon_k^i,\; i \in \mathcal{I}$
    \STATE $\mathcal{D}_{k}=\mathcal{D}_{k-1} \cup\left\{\left(\mathbf{x}_k, \mathbf{y}_k\right)\right\}$
  \ENDFOR
  \end{algorithmic}
  \caption{TVS{\small AFE}O{\small PT}}
  \label{alg: tv_safeopt}
\end{algorithm}

\subsection{Safety Guarantee}
\label{subsec:safety_guarantees}

To provide safety guarantees, we need the confidence intervals in \eqref{def: Q_t} to contain the safety functions with high probability for all iterations. Note that the parameter $\beta_k$ in \eqref{def: Q_t} tunes the tightness of the confidence interval. The following lemma guides us to make a proper choice for $\beta_k$: This choice depends on the information capacity $\gamma_k^h$ associated with the kernel $\kappa$, namely is the maximal mutual information \cite{cover1999elements} we can obtain from the GP model of $h$ through $k$ noisy measurements $\hat{h}_{\mathbf{X}_k}$ at data points $\mathbf{X}_k := \{(\mathbf{x}_\tau \in \mathcal{X}, \tau, i_\tau \in \mathcal{I})\}_{\tau <k}$, namely

\begin{equation} \label{def: gamma}
    \gamma_k^h:=\max \limits_{\mathbf{X}_k} I(\hat{h}_{\mathbf{X}_k} ; h).
\end{equation}

\begin{lemma} \label{lem:confidence_interval}
    Assume that $h(\mathbf{x}, t, i)$ has RKHS norm associated with $\kappa$ bounded by $B$ and that measurements are perturbed by $\sigma$-sub-Gaussian noise. Let the variable $\gamma_k^h$ be defined as in \eqref{def: gamma}. For any $\delta\in(0,1)$, let $\sqrt{\beta_k}=B+\sigma \sqrt{2 \left(\gamma^h_{k \cdot |\mathcal{I}|}+1+\ln (1 / \delta) \right)}$, then the following holds for all decisions $\mathbf{x} \in \mathcal{X}$, function indices $i \in \mathcal{I}$, and iterations $k \geq 1$ jointly with probability at least $1-\delta$:
\[\left|h(\mathbf{x}, t, i)-\mu_{k-1}(\mathbf{x}, i)\right| \leq \sqrt{\beta_k} \sigma_{k-1}(\mathbf{x}, i). \]
\end{lemma}
\begin{proof}
    This lemma is a straightforward consequence of Lemma 1 of \cite{widmer2023tuning}, a contextual extension of Lemma 4.1 of \cite{berkenkamp2021bayesian}. We can prove it by selecting time as the context and picking $\{t\}_{t \geq 1, t \in \mathbb{N}}$ as the context sequence.
\end{proof}

Lemma~\ref{lem:confidence_interval} indicates that, by selecting $\beta_k$ properly, the confidence intervals $Q_k$ will w.h.p. contain the reward function and the safety functions. Due to this, they can be leveraged to provide theoretical guarantees for safety and optimality. 

The following theorem provides a sufficient condition for safety of TVS{\small AFE}O{\small PT}. 

\begin{theorem} \label{thm:safety_tvsafeopt}
    Let Assumptions \ref{ass:regularity} - \ref{ass:S_0} hold, and let $\gamma_k^h$ be defined as in \eqref{def: gamma}. For any $\delta\in(0,1)$, let $\sqrt{\beta_k}=B+\sigma \sqrt{2 \left(\gamma^h_{k \cdot |\mathcal{I}|}+1+\ln (1 / \delta) \right)}$, then TVS{\small AFE}O{\small PT} guarantees that with probability at least $1 - \delta$, for all $i \in \mathcal{I}_c$ and for all $t \geq 0,$ and  $ \mathbf{x} \in S_k$ it holds $c_i(\mathbf{x}, t)~\geq~0$.
\end{theorem}
The proof builds on Lemma \ref{lem:confidence_interval} to show first that for all $ t \geq 0$, for all $i \in \mathcal{I}$ and for all $\mathbf{x} \in \mathcal{X},$  then $h(\mathbf{x},t, i) \in C_k(\mathbf{x},i) 
$ with high probability. Then using the recursive definition of the safe set from \eqref{def: S_t}, we obtain w.h.p. $c_i(\mathbf{x}, t)~\geq~l_k(\mathbf{x}^{\prime},i) - L_{\mathbf{x}} d(\mathbf{x}, \mathbf{x}^{\prime}) - L(t)\geq~0$, which concludes the proof. For details we refer the reader to Appendix \ref{sec:proof_safety}. 

\subsection{Near-Optimality Guarantee}
\label{subsec:optimality_guarantees}

In many safety critical real world applications, e.g. nuclear power plant operations, medical devices calibration, automated emergency response systems, the reward function is stationary most of the time. The problems are stationary until some changes happen and become stationary again when the systems reach new equilibria \cite{vogt2015detecting}. However, ensuring optimality is non-trivial even when the problem becomes stationary. Suppose the  auxiliary function \eqref{eq: surrogate} becomes stationary in a time interval $[\underline{\phi},\overline{\phi}]$, namely suppose there exist $\overline{\phi}>\underline{\phi}\geq 1$ such that $\forall t_1,t_2\in[\underline{\phi},\overline{\phi}], f(\mathbf{x},t_1)=f(\mathbf{x},t_2)=:\Bar{f}(\mathbf{x})$ and $ c(\mathbf{x},t_1)=c(\mathbf{x},t_2)=:\Bar{c}(\mathbf{x})$, so that  the optimization problem \eqref{eq: tv_opt} becomes
\begin{align} \label{eq: ti_opt}
\begin{aligned}
\max \limits _{\mathbf{x} \in \mathcal{X}} &\;\Bar{f}(\mathbf{x}) \\
\text{  subject to }&\Bar{c}_i(\mathbf{x})\geq 0,\;i \in \mathcal{I}_c.   
\end{aligned}    
\end{align}

We first define the largest safe set expanded from a set $S$ subjective to a measurement error $a$ within
\begin{itemize}
\item a single time step: \[
        R_a(S):=S \cup \left\{\mathbf{x} \in \mathcal{X} \mid \forall i \in \mathcal{I}_c, \exists \mathbf{x}_i^{\prime} \in S, s.t.\; \Bar{c}_i\left(\mathbf{x}_i^{\prime}\right)-L_\mathbf{x}d(\mathbf{x}, \mathbf{x}_i^{\prime})-a \geq 0\right\}
    \]
    \item $n$ time steps: $R_a^n(S):=\underbrace{R_a\left(R_a \ldots R_a\left(R_a\right.\right.}_{n \text { times }}(S))\dots)$
    \item arbitrary time steps: $\bar{R}_a(S):=\lim _{n \rightarrow \infty} R_a^n(S)$
\end{itemize}

We also define $\bar{L}_{\mathrm{t}}$ as an upper bound of the sum of all time Lipschitz constants, that is, $\sum \limits _{\tau = 0}^{T-1} L(\tau) \leq \bar{L}_{\mathrm{t}}$. We find it reasonable that a tight upper bound $\bar{L}_{\mathrm{t}}$ can be provided when the underlying system slowly switches to the new stationarity condition.

Given these definitions, we are now in the position to provide optimality guarantees for TVS{\small AFE}O{\small PT}. In particular, we aim at comparing the found reward value $\bar{f}({\mathbf{x}_k})$ with the optimal reward value within the largest safe set  obtained in ideal conditions, namely with no measurement error, $\bar{R}_0(S_0)$. We also aim at  providing TVS{\small AFE}O{\small PT} with an upper bound on the iterations needed to find a near-optimal solution. The following theorem states the optimality guarantee of TVS{\small AFE}O{\small PT}.

\begin{theorem} \label{thm:optimality_tvsafeopt}
    Let Assumptions \ref{ass:regularity} - \ref{ass:S_0} hold, let $\gamma_k^h$ be defined as in \eqref{def: gamma} and,  for any $\delta\in (0,1)$, let $\sqrt{\beta_k}=B+\sigma \sqrt{2 \left(\gamma^h_{k \cdot |\mathcal{I}|}+1+\ln (1 / \delta) \right)}$. Define $\hat{\mathbf{x}}_k$ as in (\ref{def: x_t_hat}), and, for any $\epsilon>0$, let $k^*(\epsilon, \delta)$ be the smallest positive integer satisfying
\[
\frac{k^*}{\beta_{k^*} \gamma^h_{k^*}} \geq \frac{b_1\left(\left|\bar{R}_0\left(S_0\right)\right|+1\right)}{\epsilon^2},
\]
where $b_1=8 / \log \left(1+\sigma^{-2}\right)$. Then, the TVS{\small AFE}O{\small PT} algorithm, applied to \eqref{eq: ti_opt}, guarantees that, with probability at least $1-\delta$, there exists $ k \leq k^*$ such that
\[ \Bar{f}\left(\hat{\mathbf{x}}_k\right) \geq \max _{\mathbf{x} \in \bar{R}_{\epsilon+\bar{L}_\mathrm{t}}\left(S_0\right)} \Bar{f}(\mathbf{x})-\epsilon.\]
\end{theorem}

The proof consists in showing a decaying upper bound of uncertainty $w_k(\mathbf{x},i)\leq \epsilon$ and exploiting local stationarity of \eqref{eq: ti_opt} to provide bounds on the expansion of the safe set $S_k$. Details can be found in Appendix \ref{sec:proof_optimality}.

\begin{figure} [ht]
  \centering
  \includegraphics[width=\linewidth]{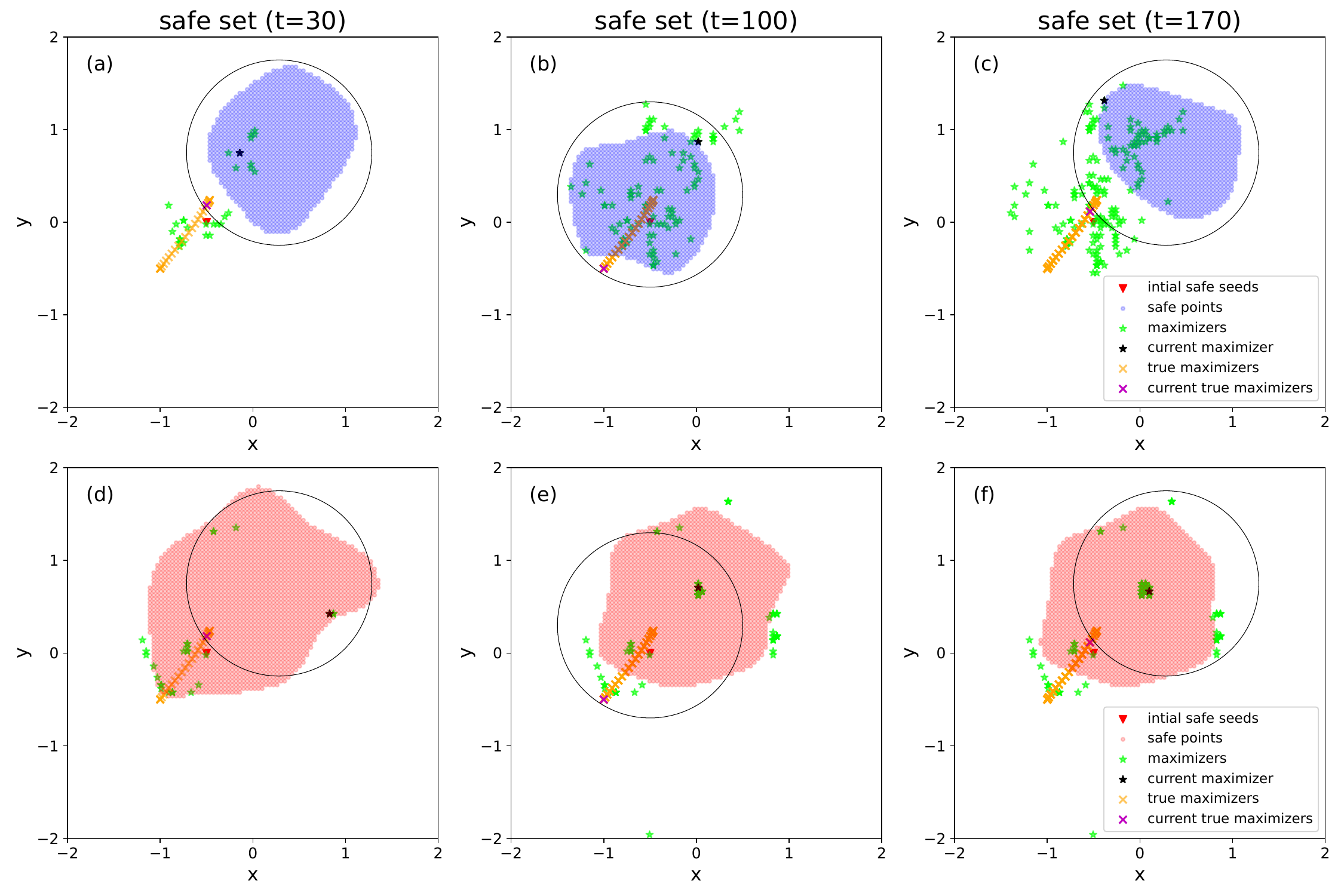}  \vspace{\baselineskip}
  \caption{Comparison of safe sets computed by TVS{\small AFE}O{\small PT} (top row) and S{\small AFE}O{\small PT} (bottom row) at $t=30, \;t=100, \text{and } t=170$. Because TVS{\small AFE}O{\small PT} takes the possible changes in time into consideration, the safe sets computed by TVS{\small AFE}O{\small PT} are contained in the ground truth safe regions while those computed by S{\small AFE}O{\small PT} have multiple violations.}
  \label{fig:toy_example_safe_set}
\end{figure}

\section{Experiments}
\label{sec:Experiments}
\subsection{Synthetic Example   }
We first illustrate TVS{\small AFE}O{\small PT} on a synthetic two-dimensional time-varying optimization problem
\begin{align*}
\max\limits_{x, y} & -e^{x^2}-\log(1+y^2)+0.01 t \\
\text { s.t. } & \left[x+0.5-0.5\left(1-\cos \frac{2 \pi}{50} t\right) \cos \frac{\pi}{6}\right]^2 + \left[y-0.3-0.5\left(1-\cos \frac{2 \pi}{50} t\right) \sin \frac{\pi}{6}\right]^2 \leq 1.    \notag 
\end{align*}

Figure~\ref{fig:toy_example_safe_set} compares the safe sets computed by TVS{\small AFE}O{\small PT} and S{\small AFE}O{\small PT} at $t=30,$ $t=100$ and $t=170$. Both algorithms start from the same singleton initial safe set $\{(-0.5, 0.0)\}$. Implementation details are described in Appendix~\ref{sec:ExperimentalDetails}. Figure~\ref{fig:toy_example_safe_set} illustrates that the safe sets computed by TVS{\small AFE}O{\small PT} are contained in the ground truth safe regions while those computed by S{\small AFE}O{\small PT} have multiple violations. Due to the dependence on time of the example (Figure~\ref{fig:optimality}), the initial safe set becomes unsafe at $t=30,\; \text{and } t=170$. Taking the possible changes in time into consideration, TVS{\small AFE}O{\small PT} correctly identifies the possible unsafety of the initial safe set. In contrast, S{\small AFE}O{\small PT} always consider the initial safe set to be safe. This toy example indicates that, in contrast to S{\small AFE}O{\small PT},  TVS{\small AFE}O{\small PT} safely adapts to the time changes of the optimization problem.

In this example, TVS{\small AFE}O{\small PT} overall finds better reward function values than S{\small AFE}O{\small PT}, see Figure~\ref{fig:optimality}. The cumulative regret of TVS{\small AFE}O{\small PT} is below that of S{\small AFE}O{\small PT} by 77.3\%. The reward function value found by TVS{\small AFE}O{\small PT} is close to the optimal values when the reward function changes slowly, which supports \Cref{thm:optimality_tvsafeopt}.

The assumed standard deviation of the noise, which is the only source of uncertainty, is 0.01 and thus leads to negligible error bars.

\begin{figure}[t]
  \centering
  \includegraphics[width=\linewidth]{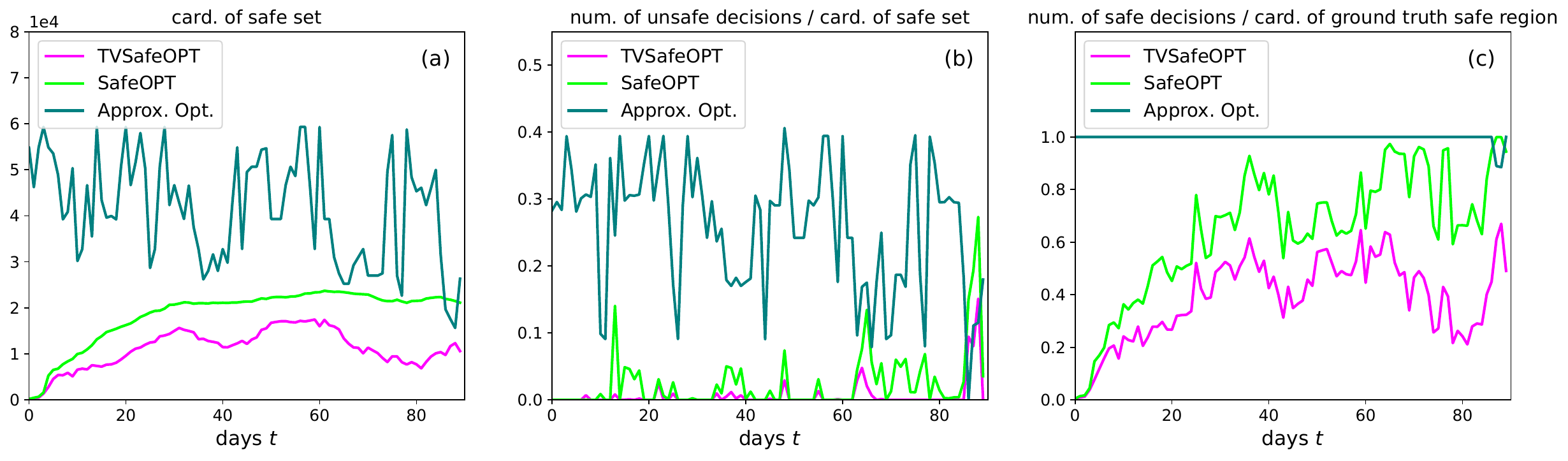}
  \vspace{\baselineskip}
  \caption{Comparison between TVS{\small AFE}O{\small PT}, S{\small AFE}O{\small PT}, and approximate optimization on the gas compressor case study. (a): The cardinality of the safe sets, (b): The ratio between the number of unsafe decisions in the safe sets and  the cardinality of the safe sets, (c): The ratio between the number of safe decisions in the safe sets and  the cardinality of the ground truth safe regions. TVS{\small AFE}O{\small PT} robustly shrinks its safe sets based on its observations and thus maintains much less violations in its safe sets than S{\small AFE}O{\small PT} and approximate optimization. It achieves this benefit at the cost of covering less of the ground truth safe region. }
  \label{fig:comparison}
\end{figure}

\subsection{Gas Compressor Case Study   }

\subsubsection{Problem Setup}
We show the performance of the proposed algorithm in a compressor station with three identical compressors operating in parallel at the time-varying compressor head $H_t$ with time-varying power consumption at time $t$ (adapted from \cite{Influence_Zagorowska2019}, details in Appendix \ref{sec:CompressorCaseSetup})
\begin{align}
    \min_{m_{i}}&{}\sum_{i=1}^N \frac{1}{1-d_{it}}\left(\alpha_1+\alpha_2\tilde{m}_{i}+\alpha_3\tilde{H}_t+\alpha_4\tilde{m}_{i}^2+\alpha_5\tilde{m}_{i}\tilde{H}_t+\alpha_6\tilde{H}_t^2 \right)&\label{eq:PowerDegraded}\\
    \text{s.t. } \sum_{i=1}^N m_{i} \geq & M_t&\label{eq:Demand}\\
    m_{i}\geq&{} \beta_1\bar{H}_t^2+\beta_2\bar{H}_t+\beta_3 \;, \forall i=1,\ldots,N&\label{eq:Surge}\\
    m_{i}\geq &{}\gamma_1\bar{\bar{H}}_t^2+\gamma_2\bar{\bar{H}}_t+\gamma_3 \;, \forall i=1,\ldots,N& \notag \\
    m_{i}\leq &{}\delta_1\tilde{\tilde{H}}_t+\delta_2 \;, \forall i=1,\ldots,N& \notag \\
    m_{i}\leq&{} \sigma_1\tilde{\bar{H}}_t^2+\sigma_2\tilde{\bar{H}}_t+\sigma_3 \;, \forall i=1,\ldots,N,&\label{eq:MaxSpeed}
\end{align}
where the objective \eqref{eq:PowerDegraded} corresponds to the power to run the station with $N$ compressors, here $N=3$, affected by individual degradation $d_{it}$, $i=1,l\dots,N$. The station must also satisfy time-varying demand $M_t$ in \eqref{eq:Demand}. In practice, it is common to linearly approximate \eqref{eq:Surge}-\eqref{eq:MaxSpeed} with respect to the compressor head $H_t$ (dashed lines in Figure~\ref{fig:linear_approx}) \cite{Experimental_Cortinovis2015}. 

\subsubsection{Results}
We compare the performance of TVS{\small AFE}O{\small PT}, S{\small AFE}O{\small PT}, and approximate optimization. Implementation details are described in Appendix~\ref{sec:ExperimentalDetails}. Figure~\ref{fig:comparison} compares the number of unsafe decisions in the safe sets calculated by TVS{\small AFE}O{\small PT}, S{\small AFE}O{\small PT}, and approximate optimization. We see that, by considering the uncertainty with respect to the decision variables, S{\small AFE}O{\small PT} maintains fewer unsafe decisions in its safe sets than the approximate optimization. However, S{\small AFE}O{\small PT} tends to expand its safe sets regardless of external changes.  TVS{\small AFE}O{\small PT} further improves this based on S{\small AFE}O{\small PT} by taking into consideration the time-varying safety functions. TVS{\small AFE}O{\small PT} robustly shrinks its safe sets based on its observations and thus maintains much less violations in its safe sets than S{\small AFE}O{\small PT} (73.9\%) and approximate optimization (97.6\%). It achieves this benefit at the cost of covering less of the ground truth safe region than S{\small AFE}O{\small PT} (40.0\%) and approximate optimization (61.4\%).

The right-hand side of Figure~\ref{fig:optimality} shows that TVS{\small AFE}O{\small PT} preserves safety at the expense of optimality. In the compressor case study, TVS{\small AFE}O{\small PT} overall finds lower reward function values than S{\small AFE}O{\small PT} and approximate optimization, which is consistent with the fact that it covers a lower fraction of the ground truth safe regions and the reward function changes significantly between iterations. Because of its strong focus on safety, TVS{\small AFE}O{\small PT} deviates more from the ground truth. The cumulative regret of TVS{\small AFE}O{\small PT} is above that of S{\small AFE}O{\small PT} by 74.9\%, and that of approximate optimization by 208.1\%. This illustrates the trade-off between safety and optimality in the presence of strong uncertainties due to the varying reward and safety constraints. 

\begin{figure} [t]
  \centering
    \begin{subfigure}[b]{0.4\textwidth}
        \centering
        \includegraphics[width=\textwidth]{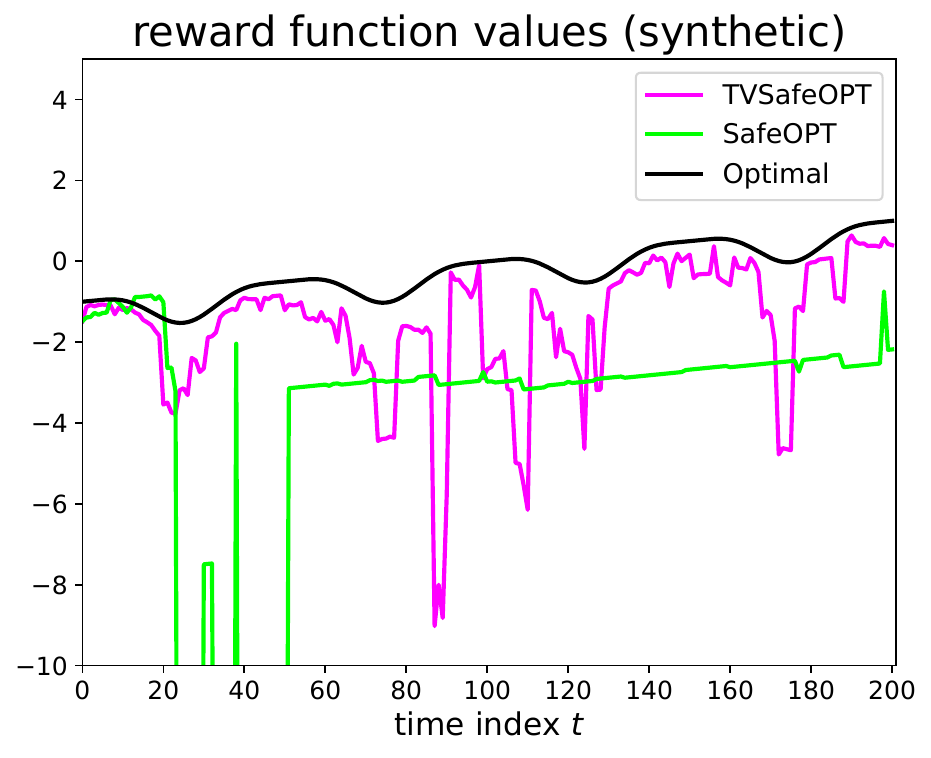}
    \end{subfigure}
    \hfill
    \begin{subfigure}[b]{0.4\textwidth}
        \centering
        \includegraphics[width=\textwidth]{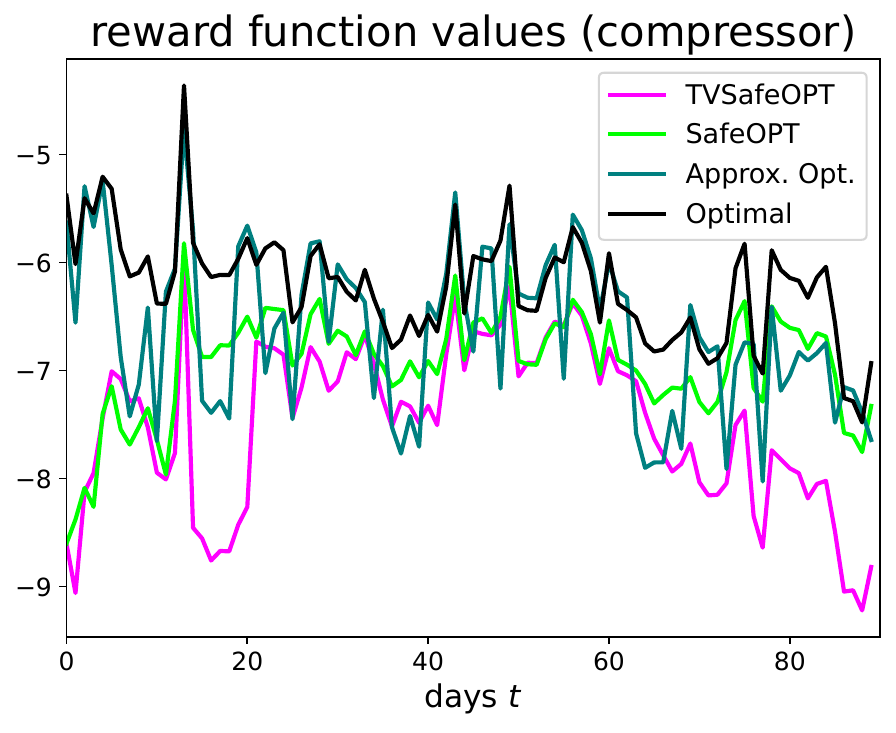}
    \end{subfigure}
    \vspace{\baselineskip}
  \caption{Reward function values found by TVS{\small AFE}O{\small PT} and S{\small AFE}O{\small PT}, compared with the optimal reward function values, in the synthetic example (left), and in the compressor case study (right). In the synthetic example, TVS{\small AFE}O{\small PT} finds better reward function values than S{\small AFE}O{\small PT}. The reward function value found by TVS{\small AFE}O{\small PT} is close to the optimum when the reward function changes slowly, which supports the theoretical result. In the compressor case study, TVS{\small AFE}O{\small PT} finds lower reward function values, since it is focused on finding robust solutions adapting to significant changes of the reward function. \label{fig:optimality}}
\end{figure}

Similarly to the synthetic example, the standard deviation of the noise is 0.01 and the error bars have been omitted.

\section{Limitations and Conclusion}
\label{sec:Limitations}
\textbf{Limitations}
\looseness -1
The compressor case study demonstrated that TVS{\small AFE}O{\small PT} ensures safety at the expense of optimality if the stationarity assumption is not satisfied. The assumption about the local stationarity of the optimization problem \eqref{eq: tv_opt} is thus the main limitation. Even though TVS{\small AFE}O{\small PT} demonstrates good empirical performance with respect to safety even when the problem is non-stationary, theoretical guarantee for its near-optimality in the non-stationary case warrant further investigation.

The need for obtaining the Lipschitz constants with respect to both $\mathbf{x}$ and $t$ in order to compute the safe set $S_k$ in \eqref{def: S_t} may prove limiting in real applications. To overcome this limitation, we propose practical modifications in \cref{sec:Practical}.

\textbf{Conclusions}
\looseness -1
We propose TVS{\small AFE}O{\small PT} algorithm, which extends S{\small AFE}O{\small PT} to deal with time-varying optimization problems. In conclusion, TVS{\small AFE}O{\small PT} outperforms S{\small AFE}O{\small PT} in terms of adaptation to changes in time and  maintains fewer unsafe decisions in its safe sets for time-varying problems. This is at the cost of covering less of the ground truth safe regions and may lead to poorer performance in terms of optimality.

We prove the safety guarantee for TVS{\small AFE}O{\small PT} in the general time-varying setting and  prove its near-optimality guarantee for the case in which the optimization problem becomes stationary. The two theoretical results together guarantee that TVS{\small AFE}O{\small PT} is capable of safely transferring safety of the decisions into the future and, based on the transferred safe sets, it will find the near-optimal decision when the reward function stops changing. We show that TVS{\small AFE}O{\small PT} performs well in practice for the most general settings where both the reward function and the safety constraint are time-varying, both on synthetic data and for real case study on a gas compressor.

\bibliography{tvsafeopt}

\begin{thebibliography}{49}
\providecommand{\natexlab}[1]{#1}
\providecommand{\url}[1]{\texttt{#1}}
\expandafter\ifx\csname urlstyle\endcsname\relax
  \providecommand{\doi}[1]{doi: #1}\else
  \providecommand{\doi}{doi: \begingroup \urlstyle{rm}\Url}\fi

\bibitem[Lizotte et~al.(2007)Lizotte, Wao, Bowling, and
  Schuurmans]{DL-RG-DS:07}
Daniel Lizotte, Tang Wao, Michael Bowling, and Dale Schuurmans.
\newblock Automatic gait optimization with {G}aussian process regression.
\newblock \emph{International Joint Conference on Artificial Intelligence},
  pages 944--949, 2007.

\bibitem[Martinez-Cantin et~al.(2007)Martinez-Cantin, de~Freitas, Douchet, and
  Castellanos]{RMC-NDF-AD-JAC:07}
Ruben Martinez-Cantin, Nando de~Freitas, Arnaud Douchet, and Jos{\'e}~A.
  Castellanos.
\newblock Active policy learning for robot planning and exploration under
  uncertainty.
\newblock \emph{In Proceedings Robotics: Science and Systems}, pages 321--328,
  2007.

\bibitem[Wang et~al.(2013)Wang, Zoghi, Hutter, Matheson, and
  de~Freitas]{ZW-BS-LJ-NDF:13}
Ziyu Wang, Masrour Zoghi, Frank Hutter, David Matheson, and Nando de~Freitas.
\newblock Bayesian optimization in high dimension via random embeddings.
\newblock \emph{International Joint Conference on Artificial Intelligence},
  pages 1778--1784, 2013.

\bibitem[Srinivas et~al.(2010)Srinivas, Krause, Kakade, and
  Seeger]{NS-AK-SMK-MS:10}
Niranjan Srinivas, Andreas Krause, Sham~M. Kakade, and Matthias Seeger.
\newblock Gaussian process optimization in the bandit setting: No regret and
  experimental design.
\newblock \emph{International Joint Conference on Machine Learning}, pages
  1015--1022, 2010.

\bibitem[Hoffman et~al.(2014)Hoffman, Shahriari, and de~Freitas]{MH-BS-NDF:14}
Matthew Hoffman, Bobak Shahriari, and Nando de~Freitas.
\newblock On correlation and budget constraints in model-based bandit
  optimization with application to automatic machine learning.
\newblock \emph{International Conference on Artificial Intelligence and
  Statistics}, pages 365--374, 2014.

\bibitem[Snoek et~al.(2012)Snoek, Larochelle, and Adams]{JS-HL-RPA:12}
Jasper Snoek, Hugo Larochelle, and Ryan~P. Adams.
\newblock Practical {B}ayesian optimization of machine learning algorithms.
\newblock \emph{In Proceedings of Advances in Neural Information Processing
  Systems}, pages 2951--2959, 2012.

\bibitem[Sui et~al.(2015)Sui, Gotovos, Burdick, and Krause]{sui2015safe}
Yanan Sui, Alkis Gotovos, Joel Burdick, and Andreas Krause.
\newblock Safe exploration for optimization with {G}aussian processes.
\newblock In \emph{International conference on machine learning}, pages
  997--1005. PMLR, 2015.

\bibitem[Berkenkamp et~al.(2021)Berkenkamp, Krause, and
  Schoellig]{berkenkamp2021bayesian}
Felix Berkenkamp, Andreas Krause, and Angela~P Schoellig.
\newblock Bayesian optimization with safety constraints: safe and automatic
  parameter tuning in robotics.
\newblock \emph{Machine Learning}, pages 1--35, 2021.

\bibitem[Sui et~al.(2018)Sui, Zhuang, Burdick, and Yue]{sui2018stagewise}
Yanan Sui, Vincent Zhuang, Joel Burdick, and Yisong Yue.
\newblock Stagewise safe {B}ayesian optimization with {G}aussian processes.
\newblock In \emph{International conference on machine learning}, pages
  4781--4789. PMLR, 2018.

\bibitem[Turchetta et~al.(2019)Turchetta, Berkenkamp, and
  Krause]{turchetta2019safe}
Matteo Turchetta, Felix Berkenkamp, and Andreas Krause.
\newblock Safe exploration for interactive machine learning.
\newblock \emph{Advances in Neural Information Processing Systems}, 32, 2019.

\bibitem[Baumann et~al.(2021)Baumann, Marco, Turchetta, and
  Trimpe]{baumann2021gosafe}
Dominik Baumann, Alonso Marco, Matteo Turchetta, and Sebastian Trimpe.
\newblock G{\small o}{S}{\small afe}: Globally optimal safe robot learning.
\newblock In \emph{2021 IEEE International Conference on Robotics and
  Automation (ICRA)}, pages 4452--4458. IEEE, 2021.

\bibitem[Sukhija et~al.(2023)Sukhija, Turchetta, Lindner, Krause, Trimpe, and
  Baumann]{sukhija2023gosafeopt}
Bhavya Sukhija, Matteo Turchetta, David Lindner, Andreas Krause, Sebastian
  Trimpe, and Dominik Baumann.
\newblock G{\small o}{S}{\small afe}{O}{\small pt}: Scalable safe exploration
  for global optimization of dynamical systems.
\newblock \emph{Artificial Intelligence}, page 103922, 2023.

\bibitem[Bottero et~al.(2024)Bottero, Luis, Vinogradska, Berkenkamp, and
  Peters]{bottero2024informationtheoretic}
Alessandro~G. Bottero, Carlos~E. Luis, Julia Vinogradska, Felix Berkenkamp, and
  Jan Peters.
\newblock Information-theoretic safe {B}ayesian optimization, 2024.

\bibitem[H{\"u}botter et~al.(2024)H{\"u}botter, Sukhija, Treven, As, and
  Krause]{hubotter2024information}
Jonas H{\"u}botter, Bhavya Sukhija, Lenart Treven, Yarden As, and Andreas
  Krause.
\newblock Information-based transductive active learning.
\newblock \emph{arXiv preprint arXiv:2402.15898}, 2024.

\bibitem[Zagorowska et~al.(2023)Zagorowska, Balta, Behrunani, Rupenyan, and
  Lygeros]{zagorowska2023efficient}
Marta Zagorowska, Efe~C. Balta, Varsha Behrunani, Alisa Rupenyan, and John
  Lygeros.
\newblock Efficient sample selection for safe learning.
\newblock \emph{IFAC-PapersOnLine}, 56\penalty0 (2):\penalty0 10107--10112,
  2023.
\newblock \doi{https://doi.org/10.1016/j.ifacol.2023.10.882}.
\newblock 22nd IFAC World Congress.

\bibitem[Widmer et~al.(2023)Widmer, Kang, Sukhija, H{\"u}botter, Krause, and
  Coros]{widmer2023tuning}
Daniel Widmer, Dongho Kang, Bhavya Sukhija, Jonas H{\"u}botter, Andreas Krause,
  and Stelian Coros.
\newblock Tuning legged locomotion controllers via safe {B}ayesian
  optimization.
\newblock In \emph{7th Annual Conference on Robot Learning (CoRL), 6-9
  November, Atlanta, GA}, 2023.
\newblock URL \url{https://openreview.net/forum?id=Tka2U40pHz0}.

\bibitem[König et~al.(2023)König, Ozols, Makarova, Balta, Krause, and
  Rupenyan]{koenig2023risk}
Christopher König, Miks Ozols, Anastasia Makarova, Efe~C. Balta, Andreas
  Krause, and Alisa Rupenyan.
\newblock Safe risk-averse {B}ayesian optimization for controller tuning.
\newblock \emph{IEEE Robotics and Automation Letters}, 8\penalty0
  (12):\penalty0 8208--8215, 2023.
\newblock \doi{10.1109/LRA.2023.3325991}.

\bibitem[Krause and Ong(2011)]{krause2011contextual}
Andreas Krause and Cheng Ong.
\newblock Contextual {G}aussian process bandit optimization.
\newblock \emph{Advances in neural information processing systems}, 24, 2011.

\bibitem[Srinivas et~al.(2009)Srinivas, Krause, Kakade, and
  Seeger]{srinivas2009gaussian}
Niranjan Srinivas, Andreas Krause, Sham~M Kakade, and Matthias Seeger.
\newblock Gaussian process optimization in the bandit setting: No regret and
  experimental design.
\newblock \emph{arXiv preprint arXiv:0912.3995}, 2009.

\bibitem[Fiducioso et~al.(2019)Fiducioso, Curi, Schumacher, Gwerder, and
  Krause]{fiducioso2019safe}
Marcello Fiducioso, Sebastian Curi, Benedikt Schumacher, Markus Gwerder, and
  Andreas Krause.
\newblock Safe contextual {B}ayesian optimization for sustainable room
  temperature {PID} control tuning.
\newblock \emph{arXiv preprint arXiv:1906.12086}, 2019.

\bibitem[K{\"o}nig et~al.(2021)K{\"o}nig, Turchetta, Lygeros, Rupenyan, and
  Krause]{konig2021safe}
Christopher K{\"o}nig, Matteo Turchetta, John Lygeros, Alisa Rupenyan, and
  Andreas Krause.
\newblock Safe and efficient model-free adaptive control via {B}ayesian
  optimization.
\newblock In \emph{2021 IEEE International Conference on Robotics and
  Automation (ICRA)}, pages 9782--9788. IEEE, 2021.

\bibitem[Bogunovic et~al.(2016)Bogunovic, Scarlett, and
  Cevher]{bogunovic2016time}
Ilija Bogunovic, Jonathan Scarlett, and Volkan Cevher.
\newblock Time-varying {G}aussian process bandit optimization.
\newblock In \emph{Artificial Intelligence and Statistics}, pages 314--323.
  PMLR, 2016.

\bibitem[Brunzema et~al.(2022)Brunzema, von Rohr, Solowjow, and
  Trimpe]{brunzema2022event}
Paul Brunzema, Alexander von Rohr, Friedrich Solowjow, and Sebastian Trimpe.
\newblock Event-triggered time-varying {B}ayesian optimization.
\newblock \emph{arXiv preprint arXiv:2208.10790}, 2022.

\bibitem[Hong et~al.(2023)Hong, Li, and Tewari]{hong2023optimization}
Kihyuk Hong, Yuhang Li, and Ambuj Tewari.
\newblock An optimization-based algorithm for non-stationary kernel bandits
  without prior knowledge.
\newblock In \emph{International Conference on Artificial Intelligence and
  Statistics}, pages 3048--3085. PMLR, 2023.

\bibitem[Zhou and Shroff(2021)]{zhou2021no}
Xingyu Zhou and Ness Shroff.
\newblock No-regret algorithms for time-varying {B}ayesian optimization.
\newblock In \emph{2021 55th Annual Conference on Information Sciences and
  Systems (CISS)}, pages 1--6. IEEE, 2021.

\bibitem[Deng et~al.(2022)Deng, Zhou, Kim, Tewari, Gupta, and
  Shroff]{deng2022weighted}
Yuntian Deng, Xingyu Zhou, Baekjin Kim, Ambuj Tewari, Abhishek Gupta, and Ness
  Shroff.
\newblock Weighted {G}aussian process bandits for non-stationary environments.
\newblock In \emph{International Conference on Artificial Intelligence and
  Statistics}, pages 6909--6932. PMLR, 2022.

\bibitem[Holzapfel et~al.(2024)Holzapfel, Brunzema, and
  Trimpe]{holzapfel2023event}
Antonia Holzapfel, Paul Brunzema, and Sebastian Trimpe.
\newblock Event-triggered safe {B}ayesian optimization on quadcopters.
\newblock In Alessandro Abate, Mark Cannon, Kostas Margellos, and Antonis
  Papachristodoulou, editors, \emph{Proceedings of the 6th Annual Learning for
  Dynamics \& Control Conference (L4DC), 15-17 July, Oxford, UK}, volume 242 of
  \emph{Proceedings of Machine Learning Research}, pages 1033--1045. PMLR,
  15--17 Jul 2024.
\newblock URL \url{https://proceedings.mlr.press/v242/holzapfel24a.html}.

\bibitem[Krishnamoorthy and Doyle(2022)]{krishnamoorthy2022safe}
Dinesh Krishnamoorthy and Francis~J Doyle.
\newblock Safe bayesian optimization using interior-point methods—applied to
  personalized insulin dose guidance.
\newblock \emph{IEEE Control Systems Letters}, 6:\penalty0 2834--2839, 2022.

\bibitem[Krishnamoorthy and Doyle~III(2023)]{krishnamoorthy2023model}
Dinesh Krishnamoorthy and Francis~J Doyle~III.
\newblock Model-free real-time optimization of process systems using safe
  bayesian optimization.
\newblock \emph{AIChE Journal}, 69\penalty0 (4):\penalty0 e17993, 2023.

\bibitem[Hewing et~al.(2020)Hewing, Wabersich, Menner, and
  Zeilinger]{hewing2020learning}
Lukas Hewing, Kim~P Wabersich, Marcel Menner, and Melanie~N Zeilinger.
\newblock Learning-based model predictive control: Toward safe learning in
  control.
\newblock \emph{Annual Review of Control, Robotics, and Autonomous Systems},
  3:\penalty0 269--296, 2020.

\bibitem[Rasmussen and Williams(2006)]{3569}
Carl~Edward Rasmussen and Chris Williams.
\newblock \emph{Gaussian Processes for Machine Learning}.
\newblock Adaptive Computation and Machine Learning. MIT Press, Cambridge, MA,
  USA, January 2006.

\bibitem[Sch{\"o}lkopf and Smola(2002)]{scholkopf2002learning}
Bernhard Sch{\"o}lkopf and Alexander~J Smola.
\newblock \emph{Learning with kernels: support vector machines, regularization,
  optimization, and beyond}.
\newblock The {MIT} Press, 2002.

\bibitem[Cover(1999)]{cover1999elements}
Thomas~M Cover.
\newblock \emph{Elements of information theory}.
\newblock John Wiley \& Sons, 1999.

\bibitem[Vogt and Dette(2015)]{vogt2015detecting}
Michael Vogt and Holger Dette.
\newblock Detecting gradual changes in locally stationary processes.
\newblock \emph{The Annals of Statistics}, 43\penalty0 (2):\penalty0 713--740,
  2015.

\bibitem[Zagorowska and Thornhill(2020)]{Influence_Zagorowska2019}
Marta Zagorowska and Nina~F. Thornhill.
\newblock Influence of compressor degradation on optimal load-sharing.
\newblock \emph{Computers and Chemical Engineering}, 143\penalty0 (5):\penalty0
  107104, 2020.

\bibitem[Cortinovis et~al.(2015)Cortinovis, Ferreau, Lewandowski, and
  Mercang{\"o}z]{Experimental_Cortinovis2015}
Andrea Cortinovis, Joachim Ferreau, Daniel Lewandowski, and Mehmet
  Mercang{\"o}z.
\newblock Experimental evaluation of {MPC}-based anti-surge and process control
  for electric driven centrifugal gas compressors.
\newblock \emph{Journal of Process Control}, 34:\penalty0 13--25, 2015.
\newblock ISSN 0959-1524.

\bibitem[Berkenkamp et~al.(2016)Berkenkamp, Schoellig, and
  Krause]{berkenkamp2016safe}
Felix Berkenkamp, Angela~P Schoellig, and Andreas Krause.
\newblock Safe controller optimization for quadrotors with {G}aussian
  processes.
\newblock In \emph{2016 IEEE International Conference on Robotics and
  Automation (ICRA), 16-21 May, Stockholm, Sweden}, pages 491--496. IEEE, 2016.

\bibitem[Kurz et~al.(2012)Kurz, Lubomirsky, and Brun]{Gas_Kurz2011}
Rainer Kurz, Matt Lubomirsky, and Klaus Brun.
\newblock Gas compressor station economic optimization.
\newblock \emph{International Journal of Rotating Machinery}, 2012:\penalty0
  Article ID 715017, 9 pages, 2012.

\bibitem[Milosavljevic et~al.(2020)Milosavljevic, Marchetti, Cortinovis,
  Faulwasser, Mercang{\"o}z, and Bonvin]{Real_Milosavljevic2020}
Predrag Milosavljevic, Alejandro~G. Marchetti, Andrea Cortinovis, Timm
  Faulwasser, Mehmet Mercang{\"o}z, and Dominique Bonvin.
\newblock Real-time optimization of load sharing for gas compressors in the
  presence of uncertainty.
\newblock \emph{Applied Energy}, 272:\penalty0 114883, 2020.

\bibitem[Zagorowska(2020)]{zagorowska2020degradation}
Marta Zagorowska.
\newblock \emph{Degradation modelling in process control applications}.
\newblock PhD thesis, Imperial College London, 2020.
\newblock available at
  \url{https://spiral.imperial.ac.uk/handle/10044/1/105173}, online 22 May
  2024.

\bibitem[Kang and Kim(2018)]{Model_Kang2018}
Do~Won Kang and Tong~Seop Kim.
\newblock Model-based performance diagnostics of heavy-duty gas turbines using
  compressor map adaptation.
\newblock \emph{Applied Energy}, 212:\penalty0 1345--1359, 2018.

\bibitem[Cortinovis et~al.(2016)Cortinovis, Mercang{\"o}z, Zovadelli, Pareschi,
  De~Marco, and Bittanti]{Online_Cortinovis2016}
Andrea Cortinovis, Mehmet Mercang{\"o}z, Matteo Zovadelli, Diego Pareschi,
  Antonio De~Marco, and Sergio Bittanti.
\newblock Online performance tracking and load sharing optimization for
  parallel operation of gas compressors.
\newblock \emph{Computers \& Chemical Engineering}, 88:\penalty0 145--156, 5
  2016.

\bibitem[Al~Zawaideh et~al.(2022)Al~Zawaideh, Al~Hosani, Boiko, and
  Hammadih]{Minimum_AlZawaideh2022}
Ayman Al~Zawaideh, Khalifa Al~Hosani, Igor Boiko, and Mohammad~Luai Hammadih.
\newblock Minimum energy adaptive load sharing of parallel operated
  compressors.
\newblock \emph{IEEE Open Journal of Industry Applications}, 3:\penalty0
  178--191, 2022.
\newblock ISSN 2644-1241.

\bibitem[N{\o}rsteb{\o}(2008)]{Optimum_Noersteboe2008}
Vibeke~Stærkebye N{\o}rsteb{\o}.
\newblock \emph{Optimum Operation of Gas Export Systems}.
\newblock PhD thesis, Norwegian University of Science and Technology, 2008.

\bibitem[Kurz and Brun(2009)]{Degradation_Kurz2009}
Rainer Kurz and Klaus Brun.
\newblock Degradation of gas turbine performance in natural gas service.
\newblock \emph{Journal of Natural Gas Science and Engineering}, 1\penalty0
  (3):\penalty0 95--102, 2009.

\bibitem[Li and Nilkitsaranont(2009)]{Gas_Li2009}
Yiuguang Li and Pannawat Nilkitsaranont.
\newblock Gas turbine performance prognostic for condition-based maintenance.
\newblock \emph{Applied Energy}, 86\penalty0 (10):\penalty0 2152--2161, October
  2009.
\newblock ISSN 0306-2619.

\bibitem[Cicciotti(2015)]{Adaptive_Cicciotti2015}
Matteo Cicciotti.
\newblock \emph{Adaptive Monitoring of Health-state and Performance of
  Industrial Centrifugal Compressors}.
\newblock PhD thesis, {I}mperial College London, 2015.

\bibitem[Zagorowska et~al.(2020)Zagorowska, Schulze~Sp\"{u}ntrup, Ditlefsen,
  Imsland, Lunde, and Thornhill]{Adaptive_Zagorowska2020}
Marta Zagorowska, Frederik Schulze~Sp\"{u}ntrup, Arne-Marius Ditlefsen, Lars
  Imsland, Erling Lunde, and Nina~F. Thornhill.
\newblock Adaptive detection and prediction of performance degradation in
  off-shore turbomachinery.
\newblock \emph{Applied Energy}, 268:\penalty0 p. 114934, 2020.

\bibitem[Chen et~al.(2021-04)Chen, Zhao, Xiang, and
  Tsoutsanis]{sequential_Chen2021}
Yu-Zhi Chen, Xu-Dong Zhao, Heng-Chao Xiang, and Elias Tsoutsanis.
\newblock A sequential model-based approach for gas turbine performance
  diagnostics.
\newblock \emph{Energy}, 220:\penalty0 119657, 2021-04.
\newblock ISSN 0360-5442.

\end{thebibliography}
\newpage
\appendix

\section{Experiment Details}
\label{sec:ExperimentalDetails}

Experiments are conducted on an Intel i7-11370H CPU using Python 3.8.5. The implementation utilizes the following libraries: GPy 1.12.0, NumPy 1.22.0, and Matplotlib 3.5.0.

\subsection{Practical Modifications}
\label{sec:Practical}
In practice, Lipschitz constants are difficult to estimate. Thus, here we provide a Lipschitz-constant-free version of TVS{\small AFE}O{\small PT} algorithm by modifying \eqref{def: S_t} and \eqref{def: e_t}.

The safe set is updated as all decisions with non-negative lower confidence bounds for the safety functions at the current iteration $k$, that is,

\begin{equation} \label{def: modified_S_t}
    S_k=\left\{\mathbf{x} \in \mathcal{X} \mid \forall i \in \mathcal{I}_c, l_k(\mathbf{x}, i) \geq 0\right\}.
\end{equation}

Furthermore, the expanders are intuitively defined as decisions within the current safe set such that, by evaluating any of the decisions, at least one decision outside the current safe set will be considered as safe, that is, $G_k=\left\{\mathbf{x} \in S_k \mid e_k(\mathbf{x})>0\right\}$, where $e_k(\mathbf{x})$ denotes the number of decisions outside $S_k$ that will be considered safe when evaluating $\mathbf{x}$. Instead of using Lipschitz constants, we define $e_k(\mathbf{x})$ using lower bound of auxiliary GP similar to the method by \citet{berkenkamp2016safe}

\begin{equation*} \label{def: modified_e_t}
    e_k(\mathbf{x})=\left|\left\{\mathbf{x}^{\prime} \in \mathcal{D} \backslash S_k \mid \exists i \in \mathcal{I}_c: l_{k,\left(\mathbf{x}, u_k(\mathbf{x}, i)\right)}\left(\mathbf{x}^{\prime}, k+1, i\right) \geq 0\right\}\right|,
\end{equation*}

where $l_{k,\left(\mathbf{x}, u_k(\mathbf{x}, i)\right)}\left(\mathbf{x}^{\prime}, k+1, i\right)$ denotes the lower bound of the function values at $\mathbf{x}$ and $t = k+1$ if $\mathbf{x}$ is evaluated at the $k$-th iteration and the upper bound is observed.

\subsection{Synthetic Example}
\label{sec:SyntheticSetup}
The search space is $\mathcal{X} = [-2, 2]^2$, uniformly quantized into $100 \times 100$ points. Both algorithms start with the singleton initial safe set $\{(-0.5,0.0)\}$. The measurements are perturbed by i.i.d. Gaussian noise $\mathcal{N}(0, 0.01^2)$. 

The reward function is formulated as: $f(\mathbf{x},t) = -e^{x^2}-\log(1+y^2)+0.01 t$;

The safety function is formulated as: $c_1(\mathbf{x},t) = 1 - \left[x+0.5-0.5\left(1-\cos \frac{2 \pi}{50} t\right) \cos \frac{\pi}{6}\right]^2 - \left[y-0.3-0.5\left(1-\cos \frac{2 \pi}{50} t\right) \sin \frac{\pi}{6}\right]^2$.

The hyperparameters of GPs in the synthetic case study are modelled as follows,

\begin{itemize}
    \item TVS{\small AFE}O{\small PT}: The reward function and the safety function are modeled by independent GPs with zero mean and spatio-temporal kernel $\kappa((\mathbf{x}, t),(\mathbf{x}^{\prime}, t^{\prime})) = \exp\left(- \frac{\|\mathbf{x} -\mathbf{x}^{\prime}\|_2^2}{2\sigma_1^2}\right) \cdot \exp\left(- (\frac{t -t^{\prime})^2}{2\sigma_2^2}\right)$, where $\sigma_1 \equiv 1.0$, $\sigma_2=25.0$ for $f$, and $\sigma_2=15.0$ for $c_1$.
    \item S{\small AFE}O{\small PT}: The reward function and the safety function are modeled by independent GPs with zero mean and 2d Gaussian kernel $\kappa(\mathbf{x},\mathbf{x}^{\prime}) = \exp\left(- \frac{\|\mathbf{x} -\mathbf{x}^{\prime}\|_2^2}{2\sigma_3^2}\right)$, where $\sigma_3 \equiv 1.0$.
\end{itemize}

\subsection{Compressor Case Study}
\label{sec:CompressorCaseSetup}

Centrifugal compressors are often used in gas transport networks to deliver the required amount of gas by boosting the pressure in the pipelines. Organised as compressor stations with $N$ units, the compressors are often operated to minimise their power consumption $P$ while satisfying the demand $M_t$ and operating constraints, capturing how compressor head $H_t$ depends on the mass flow through the compressor \cite{Gas_Kurz2011,Real_Milosavljevic2020}:
\begin{itemize}
    \item $\tilde{m}_{i}=\frac{m_{i}-157.4}{34.37}$, $\tilde{H}_t=\frac{H_t-1.016e5}{3.210e4}$, $\alpha_1=1.979e7$, $\alpha_2=5.274e6$, $\alpha_3=5.375e6$, $\alpha_4=6.055e5$, $\alpha_5=5.718e5$, $\alpha_6=3.319e5$
    \item $\bar{H}_t=\frac{H_t-1.235e5}{3.764e4}$, $\beta_1=-1.953$, $\beta_2=16.86$, $\beta_3=118.1$
    \item $\bar{\bar{H}}_t=\frac{H_t-6.152e4}{7002}$, $\gamma_1=-1.516$, $\gamma_2=-11.12$, $\gamma_3=116.9$
    \item $\tilde{\tilde{H}}_t=\frac{H_t-8.706e4}{5.289e4}$, $\delta_1=73.21$, $\delta_2=183.7$
    \item $\tilde{\bar{H}}_t=\frac{H_t-1.572e5}{2.044e4}$, $\sigma_1=-7.260$, $\sigma_2=-29.65$, $\sigma_3=204.4$
\end{itemize}
The compressor case study has been adapted from \cite{Influence_Zagorowska2019}. The data for the demand, compressor head, and degradation for the three compressors were obtained from \cite{zagorowska2020degradation} (Creative Commons Attribution NonCommercial Licence).

Individual characteristics of compressors in \eqref{eq:Surge}-\eqref{eq:MaxSpeed} are called \emph{compressor maps} (Figure~\ref{fig:linear_approx}). The operating area for a compressor is defined by minimal and maximal speed of the compressor and its mechanical properties. The operating area can be obtained from compressors maps delivered by the manufacturer of the compressor, or estimated during the operation \cite{Model_Kang2018}. However, estimation would require collecting datapoints close to the boundary of the operating area, which may be unavailable due to safety consideration \cite{Online_Cortinovis2016,Minimum_AlZawaideh2022}. Using safe learning has the potential to improve the operation of the station because it enables safe exploration of the unknown operating area of individual compressors. 

\begin{figure} [ht]
  \centering
  \includegraphics[width=0.6\linewidth]{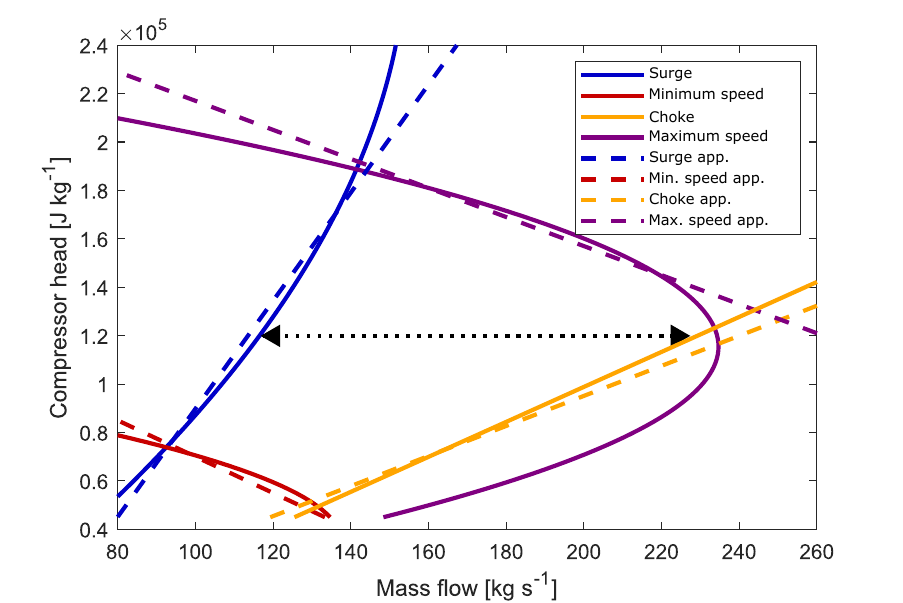}
  \vspace{\baselineskip}
  \caption{Ground truth (solid) and linear approximation (dashed) of the operating area from compressor maps, adapted from \cite{Optimum_Noersteboe2008,Influence_Zagorowska2019}. For a given compressor head at time $t$ (dotted horizontal line for $H_t=120000$ J kg$^{-1}$), the mass flow $m_{it}$ through the $i$-th compressor is required to be between minimum speed (red) and surge (blue) lines, and maximum speed (violet) and choke (yellow) lines}
  \label{fig:linear_approx}
\end{figure}

Furthermore, varying operating conditions and demand often lead to compressor degradation $d_{it}$ (Figure~\ref{fig:compressor_param}), over time increasing power consumption \eqref{eq:PowerDegraded} of the entire compressor station \cite{Degradation_Kurz2009}. Capturing the time-varying aspect of compressor degradation is a subject of research (e.g. \cite{Gas_Li2009,Adaptive_Cicciotti2015,Adaptive_Zagorowska2020}) but limited availability of measured degradation data presents a challenge \cite{sequential_Chen2021}. 

\begin{figure} [ht]
  \centering
  \includegraphics[width=\linewidth]{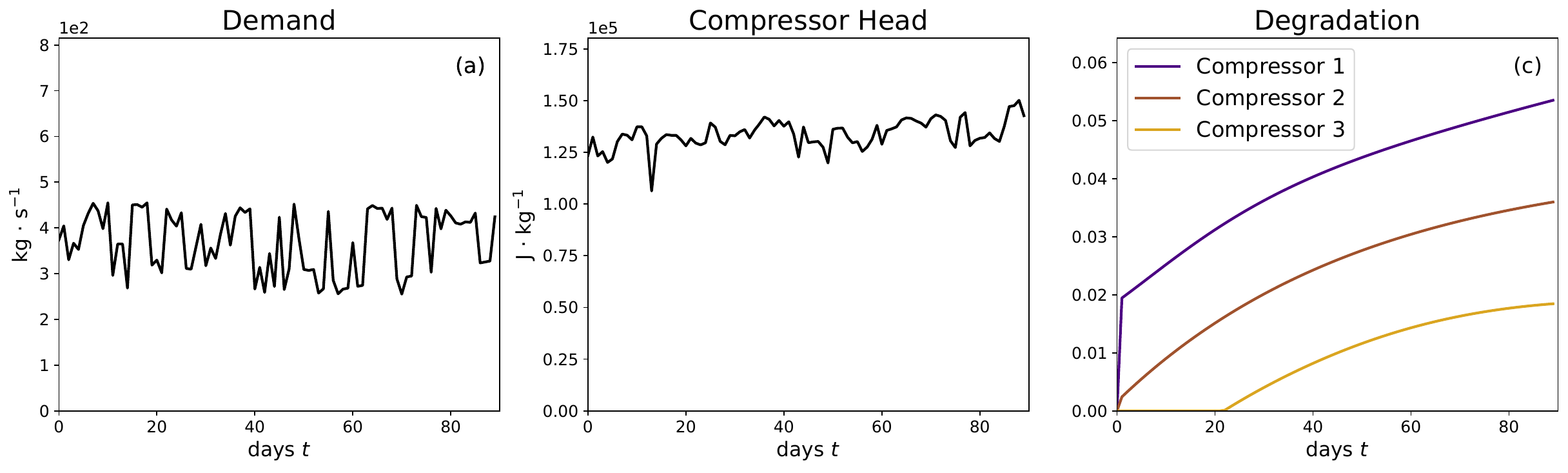}
  \vspace{\baselineskip}
  \caption{Visualization of demand (a), compressor head (b), and degradation for the compressors (c) changing with time.}
  \label{fig:compressor_param}
\end{figure}

For convenience, the optimization variables are scaled by a factor $K = 200$, that is, $\mathbf{x} = (m_1,m_2,m_3)/K$. The search space is $\mathcal{X} = [50.0/K, 250.0/K]^3$, uniformly quantized into $60 \times 60 \times 60$ points. Both algorithms start with the singleton initial safe set $\{(M_0,M_0.M_0)/3K\}$. The measurements are perturbed by i.i.d. Gaussian noise $\mathcal{N}(0, 0.01^2)$. 

The reward function is formulated as: 

\[f(\mathbf{x},t) = -\sum\limits_{i=1}^{3}\frac{1}{(1-d_{it})\cdot 10^7}\left(\alpha_1+\alpha_2\tilde{m}_{i}+\alpha_3\tilde{H}_t+\alpha_4\tilde{m}_{i}^2+\alpha_5\tilde{m}_{i}\tilde{H}_t+\alpha_6\tilde{H}_t^2 \right)\]

The safety functions are formulated as $c_i(x,t)\geq 0$, $i=1,\ldots,7$, with:
\begin{itemize}
    \item $c_1(\mathbf{x},t) = x_1 - L_t$
    \item $c_2(\mathbf{x},t) = U_t -x_1$
    \item $c_3(\mathbf{x},t) = x_2 - L_t$
    \item $c_4(\mathbf{x},t) = U_t - x_2$
    \item $c_5(\mathbf{x},t) = x_3 - L_t$
    \item $c_6(\mathbf{x},t) = U_t - x_3$
    \item $c_7(\mathbf{x},t) = x_1 + x_2 + x_3 - 0.67 M_t / K$,
\end{itemize}
where $L_t = \max\{\beta_1\bar{H}_t^2+\beta_2\bar{H}_t+\beta_3,\gamma_1\bar{\bar{H}}_t^2+\gamma_2\bar{\bar{H}}_t+\gamma_3\}/K$, $U_t = \min\{\delta_1\tilde{\tilde{H}}_t+\delta_2,\sigma_1\tilde{\bar{H}}_t^2+\sigma_2\tilde{\bar{H}}_t+\sigma_3\}/K$.

The hyperparameters of GPs in the compressor case study are modelled as follows,

\begin{itemize}
    \item TVS{\small AFE}O{\small PT}: The reward function and the safety functions are modeled by independent GPs with zero mean and spatio-temporal kernel $\kappa((\mathbf{x}, t),(\mathbf{x}^{\prime}, t^{\prime})) = \exp\left(- \frac{\|\mathbf{x} -\mathbf{x}^{\prime}\|_2^2}{2\sigma_1^2}\right) \cdot \exp\left(- (\frac{t -t^{\prime})^2}{2\sigma_2^2}\right)$, where $\sigma_1 \equiv 1.0$, $\sigma_2=80.0$ for $f$ and $c_1$ - $c_6$, and $\sigma_2=70.0$ for $c_7$.
    \item S{\small AFE}O{\small PT}: The reward function and the safety functions are modeled by independent GPs with zero mean and 3d Gaussian kernel $\kappa(\mathbf{x},\mathbf{x}^{\prime}) = \exp\left(- \frac{\|\mathbf{x} -\mathbf{x}^{\prime}\|_2^2}{2\sigma_3^2}\right)$, where $\sigma_3 \equiv 1.0$.
\end{itemize}

As for approximate optimization, the r.h.s. of \eqref{eq:Surge} - \eqref{eq:MaxSpeed} are linearly approximated as follows:

\begin{itemize}
    \item Surge line: $\beta_1\bar{H}_t^2+\beta_2\bar{H}_t+\beta_3 \approx 4.481e-4 \cdot H_t + 59.76$
    \item Min. speed line: $\gamma_1\bar{\bar{H}}_t^2+\gamma_2\bar{\bar{H}}_t+\gamma_3 \approx -1.333e-3 \cdot H_t + 193.3$
    \item Choke line: $\delta_1\tilde{\tilde{H}}_t+\delta_2 \approx 1.611e-3 \cdot H_t + 46.77$
    \item Max. speed line: $\sigma_1\tilde{\bar{H}}_t^2+\sigma_2\tilde{\bar{H}}_t+\sigma_3 \approx -1.667e-3 \cdot H_t + 461.7$
\end{itemize}

\newpage

\section{Proof of Safety Guarantee}
\label{sec:proof_safety}

Note all following lemmas hold for any  $\delta \in(0,1),$ and $S_0$,  such that $\varnothing \subsetneq S_0 \subseteq \mathcal{X}$.

First, we want to show that the intersected confidence interval $C_k$ in \eqref{def: C_t} w.h.p. contains the reward function and safety functions $h(\mathbf{x},t, i)$ as in \eqref{eq: surrogate}.

\begin{lemma} \label{lem:confidence_interval_tighter}
    Let $\sqrt{\beta_k}=B+\sigma \sqrt{2 \left(\gamma^h_{k \cdot |\mathcal{I}|}+1+\ln (1 / \delta) \right)}$, with $\gamma_k^h$ defined as in \eqref{def: gamma} and  $C_k(\mathbf{x},i)$ defined as in (\ref{def: C_t}), then the following holds with probability at least $1-\delta$ :
\[
 h(\mathbf{x},t, i) \in C_k(\mathbf{x},i) \qquad \forall t \geq 0, \forall i \in \mathcal{I}, \forall \mathbf{x} \in \mathcal{X},
\]
\end{lemma}

\begin{proof}[Proof by induction] 
\text{  }\\
If $t = 0$, by \Cref{ass:S_0} and the definition of $C_k$ in \eqref{def: C_t}, then $h(\mathbf{x},0, i) \in C_0(\mathbf{x},i),$ for all $i \in \mathcal{I}$ and for all $ \mathbf{x} \in \mathcal{X}.$

Suppose, for any $t = \tau \geq 0$, that  $h(\mathbf{x},\tau, i) \in C_\tau(\mathbf{x},i)$, then for $t = \tau + 1$,  from the Lipschitz continuity of $h$, $|h(\mathbf{x},\tau+1, i) - h(\mathbf{x},\tau, i)| \leq L(\tau)$,  it holds that $h(\mathbf{x},\tau+1, i) \in C_\tau(\mathbf{x},i) \oplus [-L(\tau), L(\tau)]$.

Moreover, by \Cref{lem:confidence_interval} and \eqref{def: Q_t} we have that  $h(\mathbf{x},\tau+1, i) \in Q_{\tau + 1}(\mathbf{x},i)$.

Thus, $h(\mathbf{x},\tau+1, i) \in \left(C_\tau(\mathbf{x},i) \oplus [-L(\tau), L(\tau)]\right) \cap Q_{\tau + 1}(\mathbf{x},i) = C_{\tau+1}(\mathbf{x},i),\; \forall i \in \mathcal{I},\; \forall \mathbf{x} \in \mathcal{X}.$

Therefore, for all $ t \geq 0$, for all $ i \in \mathcal{I}$ and for all $ \mathbf{x} \in \mathcal{X}$ we have that $ h(\mathbf{x},t, i) \in C_k(\mathbf{x},i)$, and this concludes the proof.
\end{proof}

We are now ready to prove \Cref{thm:safety_tvsafeopt} that provides a sufficient condition for TVS{\small AFE}O{\small PT} to ensure safety embedded in the constraints $c_i(\mathbf{x}, t)\geq 0$, for all $i\in \mathcal{I}_c$.

\begin{proof}[Proof of \Cref{thm:safety_tvsafeopt}]
\textit{  }\\
If $t = 0$, by definition of $S_0$, one has $c_i(\mathbf{x}, t) = c_i(\mathbf{x}, 0) \geq L(0) \geq 0$, $\forall i \in \mathcal{I}_c$, $\forall \mathbf{x} \in S_0$.

For any $t \geq 1$, $\forall \mathbf{x} \in S_{t}$, by recursive definition of $S_k$ in \eqref{def: S_t}, $\forall i \in \mathcal{I}_c$, there exists $ \mathbf{x}^{\prime} \in S_{t-1}$, $s.t. \; l_k(\mathbf{x}^{\prime},i) -L_{\mathbf{x}} d(\mathbf{x}, \mathbf{x}^{\prime}) - L(t) \geq 0$.
Then, $\forall i \in \mathcal{I}_c$
\begin{align*}
    \begin{aligned}
        &c_i(\mathbf{x}, t) \\
       \geq & c_i(\mathbf{x}^{\prime}, t) - L_{\mathbf{x}} d(\mathbf{x}, \mathbf{x}^{\prime})  \;\;\;\;\;\;\;\;\;\;\text{by Lipschitz continuity with }\mathbf{x}\\
       \geq & l_k(\mathbf{x}^{\prime},i) - L_{\mathbf{x}} d(\mathbf{x}, \mathbf{x}^{\prime}) \;\;\;\;\;\;\;\;\;\;\;\text{by \Cref{lem:confidence_interval_tighter}} \\
       \geq & l_k(\mathbf{x}^{\prime},i) - L_{\mathbf{x}} d(\mathbf{x}, \mathbf{x}^{\prime}) - L(t) \\
       \geq & 0
    \end{aligned}
\end{align*}
and this conculdes the proof.
\end{proof}

\newpage

\section{Proof of Near-Optimality Guarantee}
\label{sec:proof_optimality}
The proof of near optimality consists in two parts: i) bounding the uncertainty and ii) bounding the expansion of the safe set.

\subsection{Bounding the Uncertainty}
We first derive a decaying upper bound of uncertainty for TVS{\small AFE}O{\small PT}. In this way we can ensure the  uncertainty of the reward function and safety functions to drop below a desired threshold.

\begin{lemma} \label{lem: uncertainty_bound}
    Define $b_1 := 8 / \log \left(1+\sigma^{-2}\right) \in \mathbb{R}$, and $\gamma^h_k$ as in \eqref{def: gamma}. For any $k > k_0 \geq 1$, there exists $k^{\prime} \in (k_0, k]$, such that the following holds for all $i \in \mathcal{I}$:
\[
w_{k^{\prime}}(\mathbf{x}_{k^{\prime}},i) \leq \sqrt{\frac{b_1 \beta_k \gamma^h_{k}}{k-k_0}},
\]
\end{lemma}

\begin{proof} 
\text{    }\\
Let $i_\tau := \argmax\limits_{i \in \mathcal{I}} w_\tau(\mathbf{x}_\tau,i)$, where $\mathbf{x}_\tau = \argmax\limits_{\mathbf{x} \in G_\tau \cup M_\tau} \max\limits_{i \in \mathcal{I}} w_\tau(\mathbf{x},i)$. 
    For all $i \in \mathcal{I}$, $k_0 < k$, there exists $k^{\prime} \in (k_0, k]$:

    \begin{align*}
        \begin{aligned}
           & w_{k^{\prime}}(\mathbf{x}_{k^{\prime}},i)\\
           \leq & \frac{1}{k-k_0}\sum\limits_{\tau = k_0 + 1}^k w_{\tau}(\mathbf{x}_{\tau},i_{\tau})
        \end{aligned}
    \end{align*}
    \begin{align*}
        \begin{aligned}
            \overset{(a)}{\leq} &   \frac{2}{k-k_0}\sum\limits_{\tau = k_0 + 1}^k \sqrt{\beta_\tau} \sigma_{\tau-1}(\mathbf{x}_{\tau},i_{\tau})\\
            \leq & \frac{2\sqrt{\beta_k}}{k-k_0}\sum\limits_{\tau  =  k_0 + 1}^k  \sigma_{\tau-1}(\mathbf{x}_{\tau},i_{\tau})\\
            \overset{(b)}{\leq} & \sqrt{\frac{4 \beta_k}{k-k_0}\sum\limits_{\tau  =  k_0 + 1}^k \sigma_{\tau-1}^2(\mathbf{x}_{\tau},i_{\tau})}\\  
            \overset{(c)}{\leq} &  \sqrt{\frac{b_1 \beta_k}{k-k_0} \frac{1}{2}\sum\limits_{\tau  =  k_0 + 1}^k \log(1 + \sigma^{-2} \sigma_{\tau-1}^2(\mathbf{x}_{\tau},i_{\tau}))}\\
            \overset{(d)}{\leq} &  \sqrt{\frac{b_1 \beta_k}{k-k_0} \frac{1}{2}\sum\limits_{\tau = k_0 + 1}^k \log(1 + \sigma^{-2} \sigma_{\tau-1}^{\prime 2}(\mathbf{x}_{\tau},i_{\tau}))}\\
            \overset{(e)}{=} &  \sqrt{\frac{b_1 \beta_k I(\hat{h}_{\mathbf{X}_k}; h)}{k-k_0} } \\
            \overset{(f)}{\leq} &  \sqrt{\frac{b_1 \beta_k \gamma^h_{k}}{k-k_0}}            
        \end{aligned}
    \end{align*}
(a): Definition of $w_k$ in \eqref{def: w_t},

(b): From the fact that the quadratic mean upper bounds the arithmetic mean, 

(c): $\sigma_{\tau-1}^2(\mathbf{x}_{\tau},i_{\tau}) \leq k\left((\mathbf{x}_{\tau},\tau,i_{\tau}), (\mathbf{x}_{\tau},\tau,i_{\tau})\right) \leq 1 \text{ by \Cref{ass:regularity}},\;\text{and the fact that } a \leq \frac{b_1}{8} \log(1 + \sigma^{-2}a),\; \forall a \in [0,1]$,

(d): $\sigma^{\prime}_{\tau-1}(\mathbf{x},i)$ denotes the posterior standard deviation of $h(\mathbf{x},\tau,i)$ inferred by observations at $\mathbf{X}_\tau := \{(\mathbf{x}_j,j,i_j)\}_{j <\tau}$. Since $\{(\mathbf{x}_j,j,i_j)\}_{j = <\tau} \subsetneq \{(\mathbf{x}_j,j)\}_{j <\tau} \times \mathcal{I}$, then $\sigma_{\tau-1}(\mathbf{x}_\tau,i_\tau) \leq \sigma^{\prime}_{\tau-1}(\mathbf{x}_\tau,i_\tau)$,

(e): From \cite[Lemma 5.3]{srinivas2009gaussian},

(f): Definition of $\gamma_k^h$ \eqref{def: gamma}.

\end{proof}

\begin{corollary} \label{cor: uncertainty bound}
    Given $b_1 := 8 / \log \left(1+\sigma^{-2}\right) \in \mathbb{R}$, take $T_k$ as the smallest positive integer satisfying $\frac{T_k}{\beta_{k+T_k} \gamma^h_{k+T_k}} \geq \frac{b_1}{\epsilon^2}$. Then, there exists $k^{\prime} \in (k, k + T_k]$, such that for any $\mathbf{x} \in G_{k^{\prime}} \cup M_{k^{\prime}}$, and for all $i \in \mathcal{I}$ it holds that
\[w_{k^{\prime}}(\mathbf{x},i) \leq \epsilon.\]
\end{corollary}

\subsection{Bounding the Expansion of the Safe Set}

All following lemmas hold for any  $\delta \in(0,1), \epsilon>0$ and $S_0$,  such that $\varnothing \subsetneq S_0 \subseteq \mathcal{X}$.

To facilitate the theoretical analysis, we define $\forall \mathbf{x} \in \mathcal{X}, \forall i \in \mathcal{I}$: 

\begin{equation} \label{def: l_t_aux}
\begin{cases}
    \tilde{l}_k(\mathbf{x},i) := \max \{\tilde{l}_{k-1}(\mathbf{x},i), \mu_{k-1}(\mathbf{x},i) - \beta_k^{1/2}\sigma_{k-1}(\mathbf{x},i)\}, \quad k \geq 1\\
    \tilde{l}_0(\mathbf{x},i) := l_0(\mathbf{x},i)
\end{cases}    
\end{equation}

Remember that, from \eqref{def: C_t}, we can derive  $\forall \mathbf{x} \in \mathcal{X}, \forall i \in \mathcal{I}$:  

\begin{equation} \label{def: l_t}
   l_k(\mathbf{x},i) = \max\{l_{k-1}(\mathbf{x},i) - L(t-1),\; \mu_{k-1}(\mathbf{x},i) - \beta_k^{1/2}\sigma_{k-1}(\mathbf{x},i)\} 
\end{equation}

Therefore, $\tilde{l}_k$ can be viewed as updating $l_k$ with $L(t) \equiv 0$. With a slight abuse of notation, we omit arguments $\mathbf{x}$ and $i$ when not ambiguous.

\begin{lemma} \label{lem:basic}
    The following holds for any $k \geq 1, \forall \mathbf{x} \in \mathcal{X}, \forall i \in \mathcal{I}$:
    \begin{enumerate}
    \renewcommand{\labelenumi}{(\roman{enumi})}
        \item $l_k(\mathbf{x},i) \geq l_{k-1}(\mathbf{x},i) - L(t-1)$
        \item $\tilde{l}_k(\mathbf{x},i) \geq \tilde{l}_{k-1}(\mathbf{x},i)$
        \item $l_k(\mathbf{x},i) \leq \tilde{l}_k(\mathbf{x},i)$
        \item $\tilde{l}_k(\mathbf{x},i) - \Bar{L}_\mathrm{t} \leq l_k(\mathbf{x},i)$
    \end{enumerate}
\end{lemma}

\begin{proof}

\text{ } 
    \begin{enumerate}
    \renewcommand{\labelenumi}{(\roman{enumi})}
        \item Direct consequence of \eqref{def: l_t}.
        \item Direct consequence of \eqref{def: l_t_aux}.
        \item We proceed by induction. Suppose $l_\tau \leq \tilde{l}_\tau$, then $l_\tau - L(\tau) \leq \tilde{l}_\tau$, thus according to \eqref{def: l_t}, $l_{\tau+1} = \max\{l_{\tau} - L(\tau),\; \mu_{\tau} - \beta_{\tau+1}^{1/2}\sigma_{\tau}\} \leq \max\{\tilde{l}_\tau,\; \mu_{\tau} - \beta_{\tau+1}^{1/2}\sigma_{\tau}\} = \tilde{l}_{\tau+1}$, from which it follows $l_k(\mathbf{x},i) \leq \tilde{l}_k(\mathbf{x},i)$.
        \item We proceed by induction. Suppose $l_\tau \geq \tilde{l}_\tau  - \sum\limits_{k = 0}^{\tau-1} L(k)$. 
    
        If $\mu_{\tau} - \beta_{\tau+1}^{1/2}\sigma_{\tau} > \tilde{l}_\tau$, then  $l_{\tau + 1} \overset{\eqref{def: l_t}}{=} \mu_{\tau} - \beta_{\tau+1}^{1/2}\sigma_{\tau} \overset{\eqref{def: l_t_aux}}{=} \tilde{l}_{\tau+1} \geq \tilde{l}_{\tau+1} - \sum\limits_{k = 0}^{\tau} L(k)$.

        If $\mu_{\tau} - \beta_{\tau+1}^{1/2}\sigma_{\tau} < l_\tau - L(\tau)$, then  $l_{\tau + 1} \overset{\eqref{def: l_t}}{=} l_\tau - L(\tau) \geq \tilde{l}_\tau - \sum\limits_{k = 0}^{\tau-1} L(k) - L(\tau) = \tilde{l}_{\tau+1} - \sum\limits_{k = 0}^{\tau} L(k)$.

        Otherwise, $l_{\tau + 1} \overset{\eqref{def: l_t}}{=} \mu_{\tau} - \beta_{\tau+1}^{1/2}\sigma_{\tau} \geq l_\tau - L(\tau) \geq \tilde{l}_\tau - \sum\limits_{k = 0}^{\tau-1} L(k) - L(\tau) = \tilde{l}_{\tau+1} - \sum\limits_{k = 0}^{\tau} L(k)$.

        To summarize, $l_k \geq \tilde{l}_k  - \sum\limits_{k = 0}^{t-1} L(k) \geq \tilde{l}_k -\Bar{L}_\mathrm{t}$
    \end{enumerate} 
\end{proof}

\Cref{lem:basic} allows us to define auxiliary safe sets based on $\tilde{l}_k$ such that they are contained in $S_k$. Furthermore, due to the monotonicity of $\tilde{l}_k$, we can prove the auxiliary safe sets never shrink, which will play a fundamental role in studying their convergence property and provide near-optimality guarantee of TVS{\small AFE}O{\small PT}.

Based on \eqref{def: l_t_aux}, we further define:

$\overline{S}_k := \{\mathbf{x} \in \mathcal{X} \mid \forall i \in \mathcal{I}_c, \exists \mathbf{x}_i^{\prime} \in \overline{S}_{t-1}, s.t.\; \tilde{l}_k(\mathbf{x}_i^{\prime}, i) - L_{\mathbf{x}} d(\mathbf{x},\mathbf{x}_i^{\prime})\geq 0\}$

$\underline{S}_k := \{\mathbf{x} \in \mathcal{X} \mid \forall i \in \mathcal{I}_c, \exists \mathbf{x}_i^{\prime} \in \underline{S}_{t-1}, s.t.\; \tilde{l}_k(\mathbf{x}_i^{\prime}, i) - L_{\mathbf{x}} d(\mathbf{x},\mathbf{x}_i^{\prime}) - \bar{L}_\mathrm{t}\geq 0\}$

$\overline{S}_0 = \underline{S}_0 = S_0$

Remember $S_k = \{\mathbf{x} \in \mathcal{X} \mid \forall i \in \mathcal{I}_c, \exists \mathbf{x}_i^{\prime} \in S_{t-1}, s.t.\; l_k(\mathbf{x}_i^{\prime}, i) - L_{\mathbf{x}} d(\mathbf{x},\mathbf{x}_i^{\prime}) - L(t)\geq 0\}$. Thus, $\overline{S}_k = \underline{S}_k = S_k$ if and only if $L(t) \equiv 0$.

The following lemma proves that $\underline{S}_k$ never shrinks, and that $\underline{S}_k$ and $\overline{S}_k$ are a subset and a superset for $S_k$, respectively.

\begin{lemma} \label{lem: bound_S}
The following holds for any $t \geq 1$:
\begin{enumerate}
\renewcommand{\labelenumi}{(\roman{enumi})}
    \item $\underline{S}_{t-1} \subseteq \underline{S}_k$
    \item $\underline{S}_k \subseteq S_k \subseteq \overline{S}_k$
\end{enumerate}
    
\end{lemma}

\begin{proof} 
\text{  }\\    
\begin{enumerate}
\renewcommand{\labelenumi}{(\roman{enumi})}
    \item We refer the reader to \cite[Lemma 7.1]{berkenkamp2021bayesian}.
    \item We proceed by induction. Suppose $\underline{S}_\tau \subseteq S_\tau \subseteq \overline{S}_\tau$.

    For all $ \mathbf{x} \in S_{\tau+1}$, and for all $ i \in \mathcal{I}_c$, there exists $ \mathbf{x}_i^{\prime} \in S_{\tau} \subseteq \overline{S}_{\tau}$, s.t. $\tilde{l}_\tau(\mathbf{x}_i^{\prime}, i) - L_{\mathbf{x}} d(\mathbf{x},\mathbf{x}_i^{\prime}) \geq l_\tau(\mathbf{x}_i^{\prime}, i) - L_{\mathbf{x}} d(\mathbf{x},\mathbf{x}_i^{\prime}) - L(\tau)\geq 0$, hence $\mathbf{x} \in \overline{S}_{\tau+1}$ as well. Therefore, $S_{\tau} \subseteq \overline{S}_{\tau}$ $\forall \tau$.

    For all $\mathbf{x} \in \underline{S}_{\tau+1}$, and for all $ i \in \mathcal{I}_c$, there exists $ \mathbf{x}_i^{\prime} \in \underline{S}_\tau \subseteq S_{\tau}$, s.t. $l_\tau(\mathbf{x}_i^{\prime}, i) - L_{\mathbf{x}} d(\mathbf{x},\mathbf{x}_i^{\prime}) - L(\tau) \geq \tilde{l}_\tau(\mathbf{x}_i^{\prime}, i) - \sum\limits_{k = 0}^{\tau-1} L(k) - L_{\mathbf{x}} d(\mathbf{x},\mathbf{x}_i^{\prime}) - L(\tau) = \tilde{l}_\tau(\mathbf{x}_i^{\prime}, i) - L_{\mathbf{x}} d(\mathbf{x},\mathbf{x}_i^{\prime}) - \sum\limits_{k = 0}^{\tau} L(k) \geq \tilde{l}_\tau(\mathbf{x}_i^{\prime}, i) - L_{\mathbf{x}} d(\mathbf{x},\mathbf{x}_i^{\prime}) - \bar{L}_\mathrm{t} \geq 0$, thus $\mathbf{x} \in S_{\tau+1}$. Therefore, $\underline{S}_\tau \subseteq S_\tau$, $\forall \tau$. From which we conclude $\underline{S}_\tau \subseteq S_\tau \subseteq \overline{S}_\tau$, $\forall \tau$.
\end{enumerate}

\end{proof}

\textbf{Note:} Where needed in the following lemmas, we assume $b_1$ and $T_k$ are defined as in \Cref{lem: uncertainty_bound} and \Cref{cor: uncertainty bound}

\begin{lemma} [Lemma 7.4 in \cite{berkenkamp2021bayesian}] \label{lem: convergence_4}
    For any $k \geq 1$, $a > 0$, if $\bar{R}_a\left(S_0\right) \backslash \underline{S}_k \neq \varnothing$, then $R_a\left(\underline{S}_k\right) \backslash \underline{S}_k \neq \varnothing$.
\end{lemma}

The following lemma provides a sufficient condition for the expansion of the auxiliary safe set $\underline{S}_k$.

\begin{lemma} \label{lem: S_expansion}
    For any $t \geq 1$, if $\bar{R}_{\bar{L}_\mathrm{t} + \epsilon}(S_0) \backslash \underline{S}_k \neq \varnothing$, then, with probability at least $1-\delta$, it holds that $\underline{S}_{k+T_k} \supsetneq \underline{S}_k$.
\end{lemma}

\begin{proof}
\text{    }\\
    Similar to the proof of  \cite[Lemma 7.5]{berkenkamp2021bayesian}.
    
    By \Cref{lem: convergence_4}, we get that, $R_{\bar{L}_\mathrm{t} + \epsilon}\left(\underline{S}_k\right) \backslash \underline{S}_k \neq \varnothing$. Equivalently, $\exists \mathbf{x} \in R_{\bar{L}_\mathrm{t} + \epsilon}\left(\underline{S}_k\right) \backslash \underline{S}_k$ which implies that, for all $ i \in \mathcal{I}_c$,
    \[\exists \mathbf{z}_i \in \underline{S}_k: \Bar{c}_i(\mathbf{z}_i) - L_{\mathbf{x}} d(\mathbf{z}_i, \mathbf{x}) - \bar{L}_\mathrm{t} - \epsilon\geq 0\]
    Now assume, to the contrary, that $\underline{S}_{k+T_k} = \underline{S}_k$. Thus, $\forall k^{\prime} \in (k, k+T_k]$, $\mathbf{x} \in \mathcal{D} \backslash \underline{S}_{k^{\prime}}$, and $\forall i \in \mathcal{I}_c$, $\mathbf{z}_i \in \underline{S}_{k^{\prime}}$.
    \begin{align*}
        \begin{aligned}
           &u_{k^{\prime}}(\mathbf{z}_i,i) - L_{\mathbf{x}} d(\mathbf{z}_i, \mathbf{x})-L(k^{\prime})\\
           \geq & \Bar{c}_i(\mathbf{z}_i) - L_{\mathbf{x}} d(\mathbf{z}_i, \mathbf{x}) -L(k^{\prime})\;\;\;\;\;\;\; \text{by \Cref{lem:confidence_interval_tighter}} \\
           \geq & \Bar{c}_i(\mathbf{z}_i) - L_{\mathbf{x}} d(\mathbf{z}_i, \mathbf{x})  - \bar{L}_\mathrm{t} - \epsilon \\
           \geq & 0
        \end{aligned}
    \end{align*}
    Therefore, by definition \eqref{def: e_t}, $e_{k^{\prime}}(\mathbf{z}_i) > 0$, which implies $\mathbf{z}_i \in G_{k^{\prime}},\;\forall k^{\prime} \in (k, k+T_k],\; \forall i \in \mathcal{I}_c$.
    
    Therefore, we know that there exists $k^{\prime} \in (k,k+T_k]$, for all $i \in \mathcal{I}_c,\;w_{k^{\prime}}(\mathbf{z}_i,i) \leq \epsilon$. (\Cref{cor: uncertainty bound}) Hence, for all $i \in \mathcal{I}_c$,
    \begin{align*}
        \begin{aligned}
            & \tilde{l}_{k^{\prime}}(\mathbf{z}_i,i) - L_{\mathbf{x}} d(\mathbf{z}_i, \mathbf{x})\\
            \geq & \Bar{c}_i(\mathbf{z}_i) - w_{k^{\prime}}(\mathbf{z}_i,i) - L_{\mathbf{x}} d(\mathbf{z}_i, \mathbf{x}) \;\;\;\;\;\;\; \text{by \Cref{lem:confidence_interval_tighter}}\\
            \geq & \Bar{c}_i(\mathbf{z}_i) - \epsilon - L_{\mathbf{x}} d(\mathbf{z}_i, \mathbf{x})\\
            \geq & \bar{L}_\mathrm{t}
        \end{aligned}
    \end{align*}
    This means $\mathbf{x} \in \underline{S}_{k^{\prime}} = \underline{S}_k$, which leads to a contradiction.\\
\end{proof}

The following lemma gives a superset for the auxiliary safe set $\underline{S}_k$.

\begin{lemma} \label{lem: S_contained}
$\underline{S}_k \subseteq \Bar{R}_{\bar{L}_\mathrm{t}}(S_0)$ with probability at least $1-\delta$.   
\end{lemma}

\begin{proof} [Proof by induction]
\text{  }\\   
$\underline{S}_0 = S_0 \subseteq \Bar{R}_{\bar{L}_\mathrm{t}}(S_0)$

Suppose $\underline{S}_\tau \subseteq \Bar{R}_{\bar{L}_\mathrm{t}}(S_0)$.

For all $ \mathbf{x} \in \underline{S}_{\tau+1}$ and for all $ i \in \mathcal{I}_c$ there exists $ \mathbf{x}_i^{\prime} \in \underline{S}_\tau, s.t. \; \Bar{c}_i\left(\mathbf{x}_i^{\prime}\right)-L_\mathbf{x}d(\mathbf{x}, \mathbf{x}_i^{\prime}) - \bar{L}_\mathrm{t} \overset{(a)}{\geq} \tilde{l}_k(\mathbf{x}_i^{\prime}, i) - L_{\mathbf{x}} d(\mathbf{x},\mathbf{x}_i^{\prime}) - \bar{L}_\mathrm{t}\geq 0$.

(a): \Cref{lem:confidence_interval_tighter}.

Thus, $\underline{S}_{\tau+1} \subseteq R_{\bar{L}_\mathrm{t}}(\underline{S}_\tau) \subseteq \bar{R}_{\bar{L}_\mathrm{t}}(S_0)$
\end{proof}

\begin{lemma}[Lemma 7.8 in \cite{berkenkamp2021bayesian}]\label{lem: S_convergence} 
   Let $k^*$ be the smallest integer, such that $k^* \geq\left|\bar{R}_{\bar{L}_\mathrm{t}}\left(S_0\right)\right| T_{k^*}$. Then, there exists $k_0 \leq k^*$, such that $\underline{S}_{k_0+T_{k_0}}=\underline{S}_{k_0}$.
\end{lemma}

\Cref{lem: S_convergence} together with \Cref{lem: S_expansion}, and \Cref{lem: S_contained} entail convergence of $\underline{S}_k$ within $k^*$ time steps, which ultimately leads us to the near-optimality of TVS{\small AFE}O{\small PT} when the problem becomes stationary.

\begin{lemma} \label{lem: optimality}
       For any $k \geq 1$, if $\underline{S}_{k+T_k}=\underline{S}_k$, then,  with probability at least $1-\delta$, there exists $k^{\prime} \in (k, k+T_k]$ such that
\[
\Bar{f}\left(\hat{\mathbf{x}}_{k^{\prime}}\right) \geq \max _{\mathbf{x} \in \bar{R}_{\bar{L}_\mathrm{t} + \epsilon}\left(S_0\right)} \Bar{f}(\mathbf{x})-\epsilon.
\]
\end{lemma}

\begin{proof}
\text{    }\\

    Let $\mathbf{x}^*_{k^{\prime}}:=\argmax\limits_{\mathbf{x} \in S_{k^{\prime}}} \Bar{f}(\mathbf{x})$. Note that $\mathbf{x}^*_{k^{\prime}} \in M_{k^{\prime}}$, since
    \begin{align*}
        \begin{aligned}
            u_{k^{\prime}}(\mathbf{x}^*_{k^{\prime}},0) &\overset{(a)}{\geq}\Bar{f}(\mathbf{x}^*_{k^{\prime}})\\
            &\geq \Bar{f}(\hat{\mathbf{x}}_{k^{\prime}})\\
            &\overset{(b)}{\geq} l_{k^{\prime}}(\hat{\mathbf{x}}_{k^{\prime}},0)\\
            &\overset{(c)}{\geq} \max\limits_{\mathbf{x} \in S_{k^{\prime}}} l_{k^{\prime}}(\mathbf{x},0)
        \end{aligned}
    \end{align*}

    (a) and (b): \Cref{lem:confidence_interval_tighter},

    (c): Definition of $\hat{\mathbf{x}}_k$ \eqref{def: x_t_hat}.
    
    We will first show that $\exists k^{\prime} \in (k, k+T_k], s.t.\;\Bar{f}(\hat{\mathbf{x}}_{k^{\prime}}) \geq \Bar{f}(\mathbf{x}^*_{k^{\prime}}) - \epsilon$. Assume, to the contrary, that
    \[\forall k^{\prime} \in (k, k+T_k], \Bar{f}(\hat{\mathbf{x}}_{k^{\prime}}) < \Bar{f}(\mathbf{x}^*_{k^{\prime}}) - \epsilon\]
    Then, we have, $\exists k^{\prime} \in (k, k+T_k]$
    \begin{align*}
        \begin{aligned}
            &l_{t^{\prime}}(\mathbf{x}^*_{k^{\prime}},0) \\ \overset{(d)}{\leq}& l_{k^{\prime}}(\hat{\mathbf{x}}_{k^{\prime}},0)\\
            \overset{{(e)}}{\leq}& \Bar{f}(\hat{\mathbf{x}}_{k^{\prime}})\\
            <& \Bar{f}(\mathbf{x}^*_{k^{\prime}}) - \epsilon\\
            \overset{(f)}{\leq}&l_{k^{\prime}}(\mathbf{x}^*_{k^{\prime}},0),
        \end{aligned}
    \end{align*}
    which is a contradiction.

    (d): Definition of $\hat{\mathbf{x}}_k$ \eqref{def: x_t_hat},

    (e): \Cref{lem:confidence_interval_tighter},

    (f): \Cref{cor: uncertainty bound}, and $\mathbf{x}^*_{k^{\prime}} \in M_{k^{\prime}}$
    
Finally, $\bar{R}_{\bar{L}_\mathrm{t} + \epsilon}\left(S_0\right) \subseteq \underline{S}_{k^{\prime}} \subseteq S_{k^{\prime}}$, by \Cref{lem: S_expansion} and \Cref{lem: bound_S} (ii). Therefore, $\exists \, k^{\prime} \in (k, k+T_k]$ such that
    \begin{align*}
        \begin{aligned}
            \max _{\mathbf{x} \in \bar{R}_{\bar{L}_\mathrm{t} + \epsilon}\left(S_0\right)} \Bar{f}(\mathbf{x})-\epsilon & \leq \max _{\mathbf{x} \in S_{k^{\prime}}} \Bar{f}(\mathbf{x})-\epsilon\\
            &= \Bar{f}(\mathbf{x}^*_{k^{\prime}}) - \epsilon\\
            &\leq \Bar{f}(\hat{\mathbf{x}}_{k^{\prime}})
        \end{aligned}
    \end{align*}
\end{proof}

\subsection{Near-Optimality Proof}

\begin{proof}[Proof of \Cref{thm:optimality_tvsafeopt}]
\Cref{thm:optimality_tvsafeopt} is a direct consequence of \Cref{cor: uncertainty bound}, \Cref{lem: S_convergence}, and \Cref{lem: optimality}.
\end{proof}

\end{document}